\documentclass[10pt,twocolumn,letterpaper]{article}

\usepackage[top=0.9 in,bottom=0.9in,left=0.85in,right=0.85in,columnsep=0.4in]{geometry}

\usepackage{amsmath}
\usepackage{amssymb}
\usepackage{amsthm}
\usepackage{graphicx}
\usepackage{multirow}
\usepackage{xfrac}
\usepackage{color}

\newtheorem{definition}{Definition}
\newtheorem{assumption}{Assumption}

\newtheorem{theorem}{Theorem}
\newtheorem{corollary}{Corollary}
\newtheorem{claim}{Claim}
\newtheorem{lemma}[theorem]{Lemma}

\DeclareMathOperator*{\argmin}{arg\,min}

\newcommand{\R}{\mathbb{R}}

\newcommand{\E}{\mathbb{E}}
\renewcommand{\P}{\mathbb{P}}
\newcommand{\si}[1]{^{(#1)}}
\newcommand{\sit}[1]{^{(#1)T}}
\renewcommand{\O}{\mathcal{O}}
\renewcommand{\exp}[1]{e^{#1}}
\newcommand{\lexp}[1]{{\rm exp}(#1)}

\newcommand{\greekbf}[1]{\text{\boldmath $#1$}}
\renewcommand{\L}{\mathcal{L}}
\newcommand{\Lh}{\widehat{\L}}
\newcommand{\A}{\mathcal{A}}
\newcommand{\N}{\mathcal{N}}
\newcommand{\B}{\mathcal{B}}

\newcommand{\Lhn}{\Lh_\eta}
\newcommand{\Ln}{\L_\eta}
\newcommand{\KL}{\mathbb{KL}}

\newcommand{\Reg}{\mathcal{R}}
\newcommand{\BTheta}{\mathcal{H}}
\newcommand{\vtheta}{\greekbf{\theta}}

\newcommand{\vpsi}{\greekbf{\psi}}
\newcommand{\vphi}{\greekbf{\phi}}

\newcommand{\vthetah}{\widehat{\vtheta}}
\newcommand{\vthetahh}{\vthetah^*}

\newcommand{\vthetan}{\vtheta_\eta}
\newcommand{\vthetahn}{\vthetah_\eta}
\newcommand{\vthetahhn}{\vthetahh_\eta}
\newcommand{\veta}{\greekbf{\eta}}

\newcommand{\x}{\mathbf{x}}
\newcommand{\X}{\mathbf{X}}
\newcommand{\Xh}{\hat{\X}}
\newcommand{\Yh}{\hat{Y}}
\newcommand{\DD}{\mathcal{D}}
\newcommand{\ND}{\mathcal{Q}}

\newcommand{\tf}{\mathbf{t}}
\newcommand{\T}{\mathbf{T}}
\newcommand{\Th}{\widehat{\T}}

\newcommand{\Z}{\mathcal{Z}}

\newcommand{\Tn}{\T_\eta}
\newcommand{\Thn}{\Th_\eta}

\newcommand{\MI}{\mathbb{I}}
\newcommand{\M}{\mathcal{M}}

\newcommand{\XS}{\mathcal{X}}
\newcommand{\dotprod}[2]{\langle #1,#2 \rangle}
\newcommand{\eps}{\varepsilon}

\newcommand{\hns}{\hspace{-0.025in}}

\newcommand{\ftheta}{\theta}

\newcommand{\vnu}{\greekbf{\nu}}

\newcommand{\YS}{\mathcal{Y}}
\newcommand{\captionx}[1]{\caption{\footnotesize{#1}}} %%%
\title{\textbf{Regularized Loss Minimizers with Local Data Perturbation: Consistency and Data Irrecoverability}}

\author{
Zitao Li\\
Department of Computer Science\\
Purdue University\\
West Lafayette, IN 47907, USA\\
\texttt{li2490@purdue.edu}\\
\and
Jean Honorio\\
Department of Computer Science\\
Purdue University\\
West Lafayette, IN 47907, USA\\
\texttt{jhonorio@purdue.edu}}

\date{}

\begin{document}

\maketitle

\begin{abstract}
We introduce a new concept, \emph{data irrecoverability}, and show that the well-studied concept of data privacy is sufficient but not necessary for data irrecoverability.
We show that there are several \emph{regularized loss minimization} problems that can use perturbed data with theoretical guarantees of generalization, i.e., loss consistency.
Our results quantitatively connect the convergence rates of the learning problems to the impossibility for any adversary for recovering the original data from perturbed observations.
In addition, we show several examples where the convergence rates with perturbed data only increase the convergence rates with original data within a constant factor related to the amount of perturbation, i.e., noise.

\end{abstract}

\section{Introduction} \label{sec:intro}
In recent years, as machine learning algorithms are gradually embedded into different on-line services, there is increasing concern about privacy leakage from service providers. 
On the other hand, the enhancement of user experience and promotion of advertisement must rely on user data.
Thus, there is a natural conflict between privacy and usefulness of data. Whether data can be protected, while remaining useful, has become an interesting topic.

To resolve this conflict, several frameworks have been proposed. 
Since 2006, \emph{differential privacy} \cite{10.1007/11681878_14, dwork2009differential} has been considered as a formal definition of privacy. 
The core idea of differential privacy is to eliminate the effect of individual records from the output of learning algorithms, by introducing randomization into the process.
There is already a large number of differentially-private algorithms for different purposes \cite{dwork2008differential, wainwright2012privacy, abadi2016deep, jain2014near, bassily2014private, chaudhuri2011differentially}.
More recently, \textit{local privacy} \cite{duchi2013local, near2018differential, erlingsson2014rappor}, a stronger setting to protect individuals privacy, has been proposed. 
In local privacy, data providers randomize data before releasing it to a learning algorithm.
Locally-private algorithms related to machine learning problems have been further developed in \cite{smith2017interaction, kasiviswanathan2016efficient}.

In this paper, we discuss the effect of perturbed data on several problems in machine learning that can be modeled as the minimization of an \textit{empirical loss}, with a finite number of training samples randomly drawn from some unknown data distribution. 
In these problems, the \textit{expected loss} is usually defined as the expected value of the empirical loss, with respect to the data distribution.
The minimizers of the empirical loss and expected loss are called  the \textit{empirical minimizer} and \textit{true hypothesis} respectively.
One of the most important measurements of learning success is \textit{loss consistency}, which describes the difference between the expected loss of the empirical minimizer and that of the true hypothesis.
In \cite{honorio2014unified}, a general framework was proposed to analyze loss consistency for various problems, including the estimation of exponential family distributions, generalized linear models, matrix factorization, nonparametric regression and max-margin matrix factorization.
Additionally, in \cite{honorio2014unified} loss consistency was also used to establish other forms of consistency as corollaries of the former.
That is, loss consistency implies norm consistency (small distance between the empirical minimizer and the true hypothesis), sparsistency (recovery of the sparsity pattern of the true hypothesis) and sign consistency (recovery of the signs of the true hypothesis).

\paragraph{Contributions.}
We generalize the concept of privacy by defining the concept of \emph{data irrecoverabilitiy}.
We show that under our framework, the convergence rates of several learning problems with perturbed data, are similar to the convergence rates with original data. 
More specifically, our contributions can be summarized as follows.
\begin{itemize}
\item First, we define the concept of \emph{data irrecoverability}, and show that privacy implies data irrecoverability (Theorems~\ref{thm:irrecov} and \ref{thm:irrecovext}).
  In addition, Appendix~\ref{app:relation_dp} shows examples that are irrecoverable but not private.
  Thus, \emph{\textbf{privacy is sufficient but not necessary for irrecoverability.}}
\item Second, we show how perturbed data affect the loss consistency of several problems, by extending the assumptions and the framework of \cite{honorio2014unified}. 
  That is, we prove a perturbed loss consistency guarantee for regularized loss minimization (Theorem~\ref{thm:pri_lossconsistency}).
\item Third, \emph{\textbf{our framework allows us to analyze several empirical loss minimization problems}}, such as maximum likelihood estimation for exponential family distributions, generalized linear models with fixed design, exponential-family PCA, nonparametric generalized regression and max-margin matrix factorization.
\item We find that introducing noise with dimension-independent variance can make it difficult enough to recover the original data, \emph{\textbf{while only increasing the convergence rate within a constant factor}} (Theorem~\ref{rate:MLE} to \ref{thm:max-margin_priext}) with respect to the results reported in \cite{honorio2014unified}.
\end{itemize}

\section{Preliminaries}\label{sec:preliminary}

In this section, we will first formalize our definition of perturbed data and irrecoverability of perturbed data. 
Then we define the empirical loss minimization problems and our main assumptions.
\subsection{Perturbed Data and Irrecoverability} \label{subsec:privatized}
First we show a general definition of privacy which is used in both differential and local privacy.
\begin{definition}[Privacy] \label{def:dp}
   An algorithm $\M : \XS \to \Z$ satisfies $(\epsilon, \delta)$-privacy, where $\epsilon > 0$ and $\delta \in (0,1)$, if and only if for any input $x, x' \in \XS$ and $\mathcal{S} \in \sigma(\Z)$, we have 
   \begin{align*}
       \P_{\M}[\M(x)\in \mathcal{S}] \leq e^\epsilon\P_{\M}[\M(x')\in \mathcal{S}] + \delta,
   \end{align*}
   where $\P_{\M}$ denotes that the probability is over random draws made by the algorithm $\M$ , and $\sigma(\Z)$ denotes a $\sigma$-algebra on $\mathcal{Z}$.
\end{definition}
The above definition is very general.
Differential privacy assumes that $x$ and $x'$ are datasets that differ in a single data point. 
Group privacy assumes that $x$ and $x'$ are datasets that differ in several data points.
While $\M$ is a general mechanism in differential privacy, for local privacy $\M$ is a particular mechanism that adds noise to the data before releasing it to the learner.
Note that by setting $x$ and $x'$ as two arbitrary datasets differing in $\alpha n$ samples ($\alpha\in(0,1]$) and $\mathcal{S}=\{z\}$, we have $\prod_{i=1}^{n}p(z_i|x_i) \leq e^\epsilon\prod_{i=1}^{n}  p(z_i|x'_i)  + \delta$ by independence from Definition~\ref{def:dp}, which means $\forall i \in [n], p(z_i|x_i) \leq e^{\frac{\epsilon}{\alpha n}}  p(z_i|x'_i) + \frac{\delta}{n}$ is a sufficient condition to satisfy privacy.

\paragraph{Data irrecoverability.}
The definition of privacy can be considered as a forward mapping from data to the output of the algorithm. 
Here we analyze the backward mapping. 
That is, we focus on how likely the original data can be recovered from the algorithm output. Next we provide our formal definition. 
\begin{definition}[Data Irrecoverability]\label{def:nonrecover}
     For any privacy-preserving algorithm $\M : \XS \rightarrow \Z$ and any conceivable adversary $\A: \Z \rightarrow \XS$, we say that the original data $X$ is irrecoverable if the following holds
     for some constant $\gamma \in (0,1]$:
     \begin{align*}
         \inf_\A \P_{X,\M} [\A(\M(X)) \neq X] \geq \gamma.
     \end{align*}
\end{definition}
Our definition of data irrecoverability is more general than that of privacy. 
We can show that $(\epsilon,\delta)$-privacy implies data irrecoverability. 
Thus, in this case, our Definition~\ref{def:nonrecover}  is more general than Definition~\ref{def:dp}.
The following theorem uses privacy for arbitrary datasets $x$ and $x'$.
\begin{theorem}[Privacy implies data irrecoverability]\label{thm:irrecov}
    For any privacy-preserving algorithm $\M : \XS \rightarrow \Z$ that satisfies $(\epsilon,\delta)$-privacy where $\XS$ is a countably finite set, and any conceivable adversary $\A: \Z \rightarrow \XS$, data irrecoverability follows. That is:
    \begin{align*}
        \inf_\A \P_{X,\M}[\A(\M(X)) \neq X] \geq 1 - \frac{b(\epsilon,\delta) + \log 2}{H(X)},
    \end{align*}
    where $H(X)$ is the entropy of $X$ and 
    $
        b(\epsilon,\delta) = \inf_{x' \in \XS} \log\int_{z\in \Z} (e^\epsilon \P_{\M}(\M(x')=z)+\delta )dz,
    $
    provided that $H(X) > b(\epsilon,\delta) + \log 2$. 
    Note that $b$ can be understood as an infimum of a log-partition function.
\end{theorem}
(See Appendix~\ref{app:proofs} for detailed proofs.)

In our paper, logarithms are base $e$.  
Note that the term $b(\epsilon,\delta)$ depends on the amount of noisy introduced by $\M$.
Also, note that a higher entropy $H(X)$ implies a bigger difficulty for guessing $X$.

\begin{corollary} \label{cor:zero}
  For any privacy-preserving algorithm $\M : \XS \rightarrow \Z$ that satisfies $(\epsilon,0)$-privacy, and any conceivable adversary $\A: \Z \rightarrow \XS$, data irrecoverability follows. That is:
    \begin{align*}
        \inf_\A \P_{X,\M}[\A(\M(X)) \neq X] \geq 1 - \frac{\epsilon + \log 2}{H(X)},
    \end{align*}
    where $H(X)$ is the entropy of $X$, provided that $H(X) > \epsilon + \log 2$.
\end{corollary}

In the particular case of local privacy, we can capture the randomness of algorithm $\M$, by denoting $\M : \XS \times H \rightarrow \Z$, where $\M$ also takes a random parameter $\eta \in H$.
In order to quantify the noise, we denote the variance of the noise distribution $\ND$ as $\sigma_\eta^2$.

\begin{definition}[Generalized Data Irrecoverability]
Let $X,X' \in \XS$ be two datasets and let $d(X, X')$ be the number of different samples between $X$ and $X'$. 
     For any privacy-preserving algorithm $\M : \XS \rightarrow \Z$ and any conceivable adversary $\A: \Z \rightarrow \XS$, we say that the original data $X$ is irrecoverable if the following holds
     for some constant $\gamma \in (0,1]$:
     \begin{align*}
         \inf_\A \P_{X,\M} [d(\A(\M(X)), X) > t] \geq \gamma.
     \end{align*}
\end{definition}
We now state our theorem.
\begin{theorem}[Privacy implies generalized data irrecoverability]\label{thm:irrecovext}
  For any privacy-preserving algorithm $\M : \XS \rightarrow \Z$ that satisfies $(\epsilon,\delta)$-privacy where $\XS$ is a countably finite set, and any conceivable adversary $\A: \Z \rightarrow \XS$, data irrecoverability follows with a symmetric function $d: \XS \times \XS \rightarrow \mathbb{R}$. 
  That is:
    \begin{align*}
        \inf_\A \P_{X,\M}[d(\A(\M(X)), X) > t] \geq 1 - \frac{b(\epsilon,\delta) + \log 2}{\log\frac{|\XS|}{N_{\max}(t)}},
    \end{align*}
    where $H(X)$ is the entropy of $X$,
    $
        b(\epsilon,\delta) = \inf_{x' \in \XS} \log\int_{z\in \Z} (e^\epsilon \P_{\M}(\M(x')=z)+\delta )dz,
    $ and
    \begin{align*}
    N_{\max}(t) = \max_{X\in \XS}\sum_{X'\in \XS}\mathbf{1}[d(X, X')\leq t]
    \end{align*}
    is the maximum neighborhood size at radius $t$,
    provided that $\frac{|\XS|}{N_{\max}(t)} > b(\epsilon,\delta) + \log 2$. 
\end{theorem}

\subsection{(Perturbed) Empirical Loss Minimization Problems}

To formalize the empirical loss minimization problems, we define the problems as a tuple $\Pi = (\BTheta,\DD,\ND,\Lh,\Reg)$ for a hypothesis class $\BTheta$, a data distribution $\DD$, a noise distribution $\ND $, an empirical loss $\Lh$ and a regularizer $\Reg$.
For simplicity, we assume that $\BTheta$ is a normed vector space.

Let $\vtheta$ be a hypothesis such that $\vtheta \in \BTheta$. For the original empirical problem (without noise), let $\Lh(\vtheta)$ denote  the empirical loss of $n$ samples from an unknown data distribution $\DD$; and let $\L(\vtheta) = \E_{\DD}[\Lh(\vtheta)]$ denote the expected loss for data from distribution $\DD$.

Furthermore, let $\vpsi(\x,\veta)$ denote a mapping $\XS \times H \to \Z$.
Then, we let $\Lhn(\vtheta)$ denote the empirical loss of $n$ perturbed samples $\vpsi(\x\si{1},\veta\si{1}),\ldots,\vpsi(\x\si{n},\veta\si{n})$, where $\x\si{1},\ldots,\x\si{n}$ are samples from the unknown data distribution $\DD$, and $\veta\si{1}, \ldots, \veta\si{n}$ are noise from distribution $\ND$. 
Similarly, we let $\Ln(\vtheta) = \E_{\DD,\ND}[\Lhn(\vtheta)]$ denote the expected loss of perturbed data, where the expectation is taken with respect to both the data distribution $\DD$ and then noise distribution $\ND$.

Let $\Reg(\vtheta)$ be a regularizer and $\lambda_n>0$ be a penalty parameter.
The \emph{empirical minimizer} $\vthetahh$ and \emph{perturbed empirical minimizer} $\vthetahhn$ are given by
$
\vthetahh = \argmin_{\vtheta \in \BTheta}{\Lh(\vtheta) + \lambda_n\Reg(\vtheta)}$ and $
\vthetahhn = \argmin_{\vtheta \in \BTheta}{\Lhn(\vtheta) + \lambda_n\Reg(\vtheta)},
$
respectively.
We use a relaxed optimality assumption, defining an \emph{$\xi$-approximate empirical minimizer} $\vthetah$ and \emph{perturbed  $\xi$-approximate empirical minimizer} $\vthetahn$ with the following property for $\xi \geq 0$:
\begin{align}
    \Lh(\vthetah) + \lambda_n\Reg(\vthetah) & \leq \xi + \min_{\vtheta \in \BTheta}{\Lh(\vtheta) + \lambda_n\Reg(\vtheta)}, \\
    \Lhn(\vthetahn) + \lambda_n\Reg(\vthetahn) &\leq \xi + \min_{\vtheta \in \BTheta}{\Lhn(\vtheta) + \lambda_n\Reg(\vtheta)}.
\end{align}

The \emph{true hypothesis} is defined as 
$
\vtheta^* = \argmin_{\vtheta \in \BTheta}{\L(\vtheta)},
$
while the \emph{perturbed true hypothesis} is defined as 
$
\vthetan^* = \argmin_{\vtheta \in \BTheta}{\Ln(\vtheta)}.
$
The \emph{loss consistency} is defined as the upper bound of
$
	\L(\vthetah) - \L(\vtheta^*).
$
Similarly, in this paper, we define \emph{perturbed loss consistency}  as the upper bound of 
$
\L(\vthetahn) - \L(\vtheta^*).
$

\subsection{Assumptions}

Our first assumption is \emph{scaled} uniform convergence, a concept contrary to \emph{regular} uniform convergence.
Although both \emph{scaled} uniform convergence and \emph{regular} uniform convergence can be used to describe the difference between the empirical and expected loss for all $\vtheta$, \emph{regular} uniform convergence provides a bound that is the same for all $\vtheta$, while \emph{scaled} uniform convergence provides a bound that depends on a function of $\theta$.
We present the assumption formally in what follows:

\begin{assumption}[Scaled uniform convergence] \label{asm:closeloss}
Let $c : \BTheta \to [0;+\infty)$ be the scale function.
The empirical loss $\Lhn$ is close to its expected value $\Ln$, such that their absolute difference is proportional to the scale of the hypothesis $\vtheta$.
That is, with probability at least $1-\delta$ over draws of $n$ samples:
\begin{equation}
(\forall \vtheta \in \BTheta){\rm\ }\left|\Lhn(\vtheta) - \Ln(\vtheta)\right| \leq \eps_{n, \delta}c(\vtheta)
\end{equation}
\noindent where the rate $\eps_{n,\delta}$ is nonincreasing with respect to $n$ and $\delta$.
Furthermore, assume $\lim_{n \to +\infty}{\eps_{n,\delta}} = 0$ for $\delta \in (0,1)$.

\end{assumption}

Next, we borrow the  \emph{super-scale regularizers} assumption from \cite{honorio2014unified}, which defines regularizers lower-bounded by a scale function.

\begin{assumption}[Super-scale regularization \cite{honorio2014unified}] \label{asm:superreg}
Let $c : \BTheta \to [0;+\infty)$ be the scale function.
Let $r : [0;+\infty) \to [0;+\infty)$ be a function such that
$
(\forall z \geq 0){\rm\ }z \leq r(z). 
$
The regularizer $\Reg$ is bounded as
$(\forall \vtheta \in \BTheta){\rm\ }r(c(\vtheta)) \leq \Reg(\vtheta) < +\infty$ .
\end{assumption}
Note that the above assumption implies $c(\vtheta) \leq \Reg(\vtheta)$.
Next, we introduce an assumption for the difference between the expected loss for perturbed data and that of original data.

\begin{assumption}[Bounded perturbed loss]\label{asm:difference}
Let $c : \BTheta \to [0;+\infty)$ be the scale function.
The expected loss of the perturbed data $\Ln$ is close to the expected loss of the original data $\L$, such that their absolute difference is proportional to the scale of the hypothesis $\vtheta$.
That is, with draws of $n$ samples:
    \begin{eqnarray*}
		\forall \vtheta \in  \BTheta, \left|\Ln(\vtheta) - \L(\vtheta) \right|	\leq \eps'_n c(\vtheta),
	\end{eqnarray*}
\end{assumption}

\subsection{Perturbed Loss Consistency} \label{sec:theoretical}

In this part, we formally show perturbed loss consistency, a worst-case guarantee of the difference between the expected loss under the original data distribution $\DD$ of the $\xi$-approximate empirical minimizer from perturbed data, $\vthetahn$, and that of the true hypothesis $\vtheta^*$.

\begin{theorem}[Perturbed Loss consistency] \label{thm:pri_lossconsistency}
Under Assumption~\ref{asm:closeloss} with rate $\eps_{n, \delta}$, Assumption~\ref{asm:superreg} for regularizers, and Assumption~\ref{asm:difference} with rate $\eps'_n$, \emph{perturbed regularized loss minimization} is loss-consistent.
That is, for $\alpha\geq 2$, $\lambda_n = \alpha\eps_{n,\delta}$ and $\eps'_n \leq \eps_{n,\delta}$, with probability at least $1-\delta$:
	\begin{equation}
	\L(\vthetahn) - \L(\vtheta^*) \leq \eps_{n,\delta}\left(\alpha\Reg(\vthetan^*) + c(\vthetan^*)\right) +\eps'_n c(\vtheta^*) + \xi.
	\end{equation}
\end{theorem}

Based on Theorem~\ref{thm:pri_lossconsistency}, the perturbed loss consistency result maintains the same structure as the one for original data \cite{honorio2014unified}, with and additional term (i.e., $\eps'_n c(\vtheta^*)$).
In the following section, we show that the problems that we study will either have larger $\eps_{n,\delta}$ than the ones in \cite{honorio2014unified} and $\eps'_n = 0$, or have the same $\eps_{n,\delta}$ as the ones in \cite{honorio2014unified} and $\eps'_n > 0$. 
Thus, the loss consistency for perturbed data leads to a larger upper bound when compared to using original data.
Fortunately, we show that the difference is only in constant factors.

\section{Examples} \label{sec:examples}
In this section, we show that several popular problems can be analyzed with our novel framework. 
For the first four examples in Subsection~\ref{subsec:MLE} to \ref{subsec:np}, we focus on a special class of algorithms that perform \emph{unbiased data perturbation}.
In Subsection~\ref{subsec:max-margin}, we focus on an algorithm that performs a sign-flipping data perturbation.
\begin{definition}[Unbiased Data Perturbation]\label{def:unbiased}
    Let $\vpsi(\x,\veta)$ denote a mapping $\XS \times H \to \Z$, where $\x \in \XS$ is the original data sample drawn from $\DD$ and $\veta \in H$ is the noise drawn from $\ND$. 
    We say that the function $\vpsi(\x,\veta)$ is unbiased if it satisfies 
    $
        \mathbb{E}_{\ND}[\vpsi(\x,\veta)] = \tf(\x),
   $
    for all $\x \in \XS$, where $\tf(\x)$ is the sufficient statistic for a particular machine learning problem.
\end{definition}

Table~\ref{result-table} summarizes the convergence rates achieved for several examples using our proposed framework. 
Table~\ref{result-table} also shows the minimum noise variance in order to achieve data irrecoverability in the last column.
For example, we can obtain a convergence rate of $\O(\sqrt{\sigma_{\x}^2 + \sigma_{\veta}^2}\sqrt{\log{\frac{1}{\delta}}}\sqrt{\frac{\log p}{n}})$ for maximum likelihood estimation for exponential family distribution with $\ell_1$ regularizer, sub-Gaussian sufficient statistics with variance  $\sigma_{\x}$, and perturbation/noise distribution with variance $\sigma_{\veta}$.
Meanwhile, if the perturbation/noise distribution has variance at least $\sigma_\eta^2 \geq \frac{4}{(1-\gamma)\log 2 }$, then any adversary will fail to recover the original data up to permutation with probability greater than $\gamma$.
Thus, the introduced noise has dimension-independent variance, which guarantees irrecoverability, while only increasing the convergence rate within a constant factor with respect to \cite{honorio2014unified}, from $\sigma_{\x}$ to $\sqrt{\sigma^2_{\x} + \sigma^2_{\veta}}$.
\begin{table*}[ht]
	\captionx{New Convergence Rates $\eps_{n, \delta}$ with Data Irrecoverability and Minimum Noise for Examples in Section~\ref{sec:examples}, Theorem~\ref{rate:MLE} to \ref{thm:max-margin_priext}. 
    For example, we can obtain a convergence rate of $\O(\sqrt{\sigma_{\x}^2 + \sigma_{\veta}^2}\sqrt{\log{\frac{1}{\delta}}}\sqrt{\frac{\log p}{n}})$ for maximum likelihood estimation for exponential family distribution with $\ell_1$ regularizer, sub-Gaussian sufficient statistics with variance  $\sigma_{\x}$, and perturbation/noise distribution with variance $\sigma_{\veta}$.
    Meanwhile, if the perturbation/noise distribution has variance at least $\sigma_\eta^2 \geq \frac{4}{(1-\gamma)\log 2 }$, then any adversary will fail to recover the original data up to permutation with probability greater than $\gamma$.
    Thus, the introduced noise has dimension-independent variance, which guarantees irrecoverability, while only increasing the convergence rate within a constant factor with respect to \cite{honorio2014unified}, from $\sigma_{\x}$ to $\sqrt{\sigma^2_{\x} + \sigma^2_{\veta}}$.
	}
	\label{result-table}
	\begin{center}
		\begin{footnotesize} %%%
			\resizebox{\textwidth}{!}{%
			\begin{tabular}{@{}p{1 in}@{\hspace{0.1in}}l@{\hspace{0.05in}}c@{\hspace{0.05in}}c@{\hspace{0.05in}}c@{\hspace{0.05in}}c@{\hspace{0.05in}}c@{\hspace{0.05in}}c@{\hspace{0.05in}}c@{\hspace{0.05in}}c@{}}
				\hline
				\vspace{-1.3in}\parbox{2.5in}{
				The convergence rates $\eps_{n, \delta}$ are for $n$ samples with respect to $p$-dimension sufficient statistics, i.e., $\vtheta \in \BTheta = \R^p$ (for exponential-family PCA, $\vtheta \in \BTheta = \R^{n_1\times n_2}$ and $n = n_1\times n_2$), with probability at least $1-\delta$.
                $\beta \in (0,1/2)$ is a parameter for nonparametric regression.
                $\sigma_x$ and $\sigma_\eta$ are the parameters of sub-Gaussian distributions or maximum variances as described in Lemma~\ref{lemma:subgaussian} and \ref{lemma:finitevariance}. 
                Rates were not optimized. 
                All rates follow from the specific regularizer and norm inequalities. 
                NA means "not applicable" and NG means "no guarantees" in the table.}
				& 
				& \rotatebox{90}{\parbox{1.3in}{\scriptsize Sparsity ($\ell_1$) \cite{Ravikumar08} \\ Elastic net \cite{Zou05} \\ Total variation \cite{Zhang10} \\ Sparsity and low-rank \cite{Richard12} \\ Quasiconvex ($\ell_1\hns+\hns\ell_p,p\hns<\hns 1$)}}
				& \rotatebox{90}{\parbox{1.3in}{\scriptsize Sparsity \cite{Argyriou12} \\ ($k$-support norm)}}
				& \rotatebox{90}{\parbox{1.3in}{\scriptsize Tikhonov \cite{Hsu12}\\ Multitask ($\ell_{1,\infty})$ \\ Dirty multitask}}
				& \rotatebox{90}{\parbox{1.3in}{\scriptsize Multitask ($\ell_{1,2}$) \cite{Jacob09}}}
				& \rotatebox{90}{\parbox{1.3in}{\scriptsize Overlap multitask ($\ell_{1,2}$) \cite{Jacob09}\\ $g$ is maximum group size}}
				& \rotatebox{90}{\parbox{1.3in}{\scriptsize Overlap multitask ($\ell_{1,\infty}$) \cite{Mairal10} \\ $g$ is maximum group size}}
				& \rotatebox{90}{\parbox{1.25in}{\scriptsize Low-rank \cite{Richard12}}}
				& \rotatebox{90}{\parbox{1.25in}{\scriptsize Minimum noise to make data reconstruction impossible}} 
				\\
				\hline
				\multirow{2}{*}{\shortstack[l]{MLE for exponential \\family distribution }} &sub-Gaussian  ($\sqrt{\sigma_\x^2+\sigma_\eta^2}\sqrt{\log\sfrac{1}{\delta}}$) 
				& $\sqrt{\frac{\log p}{n}}$ 
				& $\sqrt{\frac{k\log p}{n}}$ 
				& $\sqrt{\frac{p\log p}{n}}$ 
				& $\frac{p^{1/4}\sqrt{\log p}}{\sqrt{n}}$ 
				& $\sqrt{\frac{g\log p}{n}}$ 
				& $\frac{g\sqrt{\log p}}{\sqrt{n}}$ 
				& $\sqrt{\frac{p\log p}{n}}$ 
				& \multirow{2}{*}{$\sigma_\eta^2 \geq \frac{4}{(1-\gamma)\log 2 }$ }
				\\
				& Finite variance ($\sqrt{\sigma_\x^2+\sigma_\eta^2}\sqrt{\sfrac{1}{\delta}}$) 
				& $\sqrt{\frac{p}{n}}$ 
				& $\sqrt{\frac{kp}{n}}$ 
				& $\frac{p}{\sqrt{n}}$ 
				& $\frac{p^{3/4}}{\sqrt{n}}$ 
				& $\sqrt{\frac{gp}{n}}$ 
				& $\frac{g\sqrt{p}}{\sqrt{n}}$ 
				& $\frac{p}{\sqrt{n}}$ 
				\\
				\multirow{2}{*}{\shortstack[l]{GLM with \\fixed design}} 
				&sub-Gaussian ($\sqrt{\sigma_y^2+\sigma_\eta^2}\sqrt{\log\sfrac{1}{\delta}}$) 
				& $\sqrt{\frac{\log p}{n}}$ 
				& $\sqrt{\frac{k\log p}{n}}$ 
				& $\sqrt{\frac{p\log p}{n}}$ 
				& $\frac{p^{1/4}\sqrt{\log p}}{\sqrt{n}}$ 
				& $\sqrt{\frac{g\log p}{n}}$ 
				& $\frac{g\sqrt{\log p}}{\sqrt{n}}$ 
				& NA 
				& \multirow{2}{*}{$\sigma_\eta^2 \geq \frac{2}{(1-\gamma)\log 2 }$ }
				\\
				&Finite variance ($\sqrt{\sigma_y^2+\sigma_\eta^2}\sqrt{\sfrac{1}{\delta}}$) 
				& $\sqrt{\frac{p}{n}}$ 
				& $\sqrt{\frac{kp}{n}}$ 
				& $\frac{p}{\sqrt{n}}$ 
				& $\frac{p^{3/4}}{\sqrt{n}}$ 
				& $\sqrt{\frac{gp}{n}}$ 
				& $\frac{g\sqrt{p}}{\sqrt{n}}$ 
				& NA 
				\\	\multirow{2}{*}{\shortstack[l]{Exponential-family \\PCA}}
				& sub-Gaussian ($\sqrt{\sigma_\x^2+\sigma_\eta^2}\sqrt{\log\sfrac{1}{\delta}}$) 
				& $\frac{\sqrt{\log n}}{n}$ 
				& NA 
				& $\sqrt{\frac{\log n}{n}}$ 
				& $\frac{\sqrt{\log n}}{n^{3/4}}$ 
				& NA 
				& NA 
				& $\sqrt{\frac{\log n}{n}}$ 
				& \multirow{2}{*}{$\sigma_\eta^2 \geq \frac{2}{(1-\gamma)\log 2 }$ }
				\\
				&Finite variance ($\sqrt{\sigma_\x^2+\sigma_\eta^2}\sqrt{\sfrac{1}{\delta}}$) 
				& $\frac{1}{\sqrt{n}}$ 
				& NA 
				& NG 
				& $\frac{1}{n^{1/4}}$ 
				& NA 
				& NA 
				& NG 
				\\
			\multirow{2}{*}{\shortstack[l]{Nonparametric \\regression}} 
			&sub-Gaussian ($\sqrt{\sigma_y^2+\sigma_\eta^2}\sqrt{\log\sfrac{1}{\delta}}$) 
			& $\frac{\sqrt{\log p}}{n^{1/2-\beta}}$ 
			& $\frac{\sqrt{k\log p}}{n^{1/2-\beta}}$ 
			& $\frac{p\sqrt{\log p}}{n^{1/2-\beta}}$ 
			& $\frac{\sqrt{p\log p}}{n^{1/2-\beta}}$ 
			& $\frac{\sqrt{g\log p}}{n^{1/2-\beta}}$ 
			& $\frac{g\sqrt{\log p}}{n^{1/2-\beta}}$ 
			& NA 
			& \multirow{2}{*}{$\sigma_\eta^2 \geq \frac{2}{(1-\gamma)\log 2 }$ }
			\\
				&Finite variance ($\sqrt{\sigma_y^2+\sigma_\eta^2}\sqrt{\sfrac{1}{\delta}}$) 
				&$\frac{\sqrt{p}}{n^{1/2-\beta}}$ 
				&$\frac{\sqrt{kp}}{n^{1/2-\beta}}$ 
				&$\frac{p^{3/2}}{n^{1/2-\beta}}$ 
				&$\frac{p}{n^{1/2-\beta}}$ 
				\\
			\shortstack[l]{Max-margin matrix\\factorization} & ($\delta=0$) & $\frac{1}{n}$ &NA &$\frac{1}{\sqrt{n}}$ & $\frac{1}{n^{3/4}}$ 
			&NA 
			&NA 
			&$\frac{1}{\sqrt{n}}$ 
            & $q \in (\frac{1}{2}, \frac{1}{2} + \frac{(1-\gamma) \log 2}{8})$
            \\
			\hline
			\end{tabular}
		}
		\end{footnotesize} %%%
	\end{center}
	\vspace{-0.1in}
\end{table*}

Several regularizers are shown to fulfill Assumption~\ref{asm:superreg} in  \cite{honorio2014unified}. 
Norm regularizers such as the $\ell_1$-norm \cite{Ravikumar08}, the $k$-support norm \cite{Argyriou12}, the multitask $\ell_{1,2}$ and $\ell_{1,\infty}$-norms \cite{Jacob09,Mairal10,Negahban11,Obozinski11}, and the trace norm \cite{Bach08,Srebro04} fulfill Assumption~\ref{asm:superreg} with $c(\vtheta) = \|\vtheta\|$ and $r(z) = z$.
The Tikhonov regularizer \cite{Hsu12} fulfills Assumption~\ref{asm:superreg} with $c(\vtheta) = \|\vtheta\|_2$ and $r(z) = z^2 + \frac{1}{4}$.
Other regularizers such as the low-rank prior \cite{Richard12}, the elastic net \cite{Zou05}, dirty models \cite{Jalali10} and the total variation prior \cite{Zhang10} also fulfill Assumption~\ref{asm:superreg}.

Before discussing various examples, we present two technical lemmas that are useful for the analysis of the perturbed loss consistency.
\begin{lemma}\label{lemma:subgaussian}
	Given the sufficient statistic $\tf(\x)$. 
	Assume that $\forall j, \tf_j(\x)$ follows a sub-Gaussian distribution with parameter $\sigma_x$, 
	and that the conditional distribution of $\vpsi_j(\x,\veta)$ for any fixed $\x$ is  sub-Gaussian with parameter $\sigma_\eta$. 
	We have that $\vpsi_j(\x,\veta)$ follows a sub-Gaussian distribution with parameter $\sigma$, such that $\sigma^2 = \sigma_x^2 + \sigma_\eta^2$.
\end{lemma}
\begin{lemma}\label{lemma:finitevariance}
	Given the sufficient statistic $\tf(\x)$. 
	Assume that $\forall j, \tf_j(\x)$ has variance at most $\sigma_x^2$, 
	and that the conditional distribution of $\vpsi_j(\x,\veta)$ for any fixed $\x$ has variance at most $\sigma_\eta^2$. 
	We have that $\vpsi(\x,\veta)$ has variance at most $\sigma^2 = \sigma_x^2 + \sigma_\eta^2$.
\end{lemma}

\subsection{Maximum Likelihood Estimation for Exponential Family Distributions}\label{subsec:MLE}

First, we focus on the problem of maximum likelihood estimation(MLE) for exponential family distributions \cite{Kakade10,Ravikumar08} with arbitrary norms regularization. 
This includes for instance, the problem of learning the parameters (and possibly structure) of Gaussian and discrete MRFs.
We provide a new convergence rate $\eps_{n,\delta}$ with perturbed data and provide an impossibility result for the recovery of the original data.

To define the problem, let $\tf(\x)$ be the sufficient statistic and $\Z(\vtheta) = \int_\x{\exp{\dotprod{\tf(\x)}{\vtheta}}}$ be the partition function.
Given $n$ i.i.d. samples, let $\Th = \frac{1}{n}\sum_i{\tf(\x\si{i})}$ be the original empirical sufficient statistics, and let $\T = \E_{\x \sim \DD}[\tf(\x)]$ be the expected sufficient statistics. 
After we perturb the $n$ samples, denote $\Thn = \frac{1}{n}\sum_i{\vpsi(\x\si{i},\veta\si{i})}$ as the empirical statistics for perturbed data, and $\Tn = \E_{\x \sim \DD, \veta \sim \ND}[\vpsi(\x,\veta)]$ as the expected sufficient statistic after perturbation.
Let $\Lh(\vtheta) = -\dotprod{\Th}{\vtheta} + \log\Z(\vtheta)$ be the empirical negative log-likelihood for original data $\Th$.
Let $\Lhn(\vtheta) = -\dotprod{\Thn}{\vtheta} + \log\Z(\vtheta)$ be the empirical negative log-likelihood for privatized data $\Thn$.
Similarly, $\L(\vtheta) = -\dotprod{\T}{\vtheta} + \log\Z(\vtheta)$ and $\Ln(\vtheta) = -\dotprod{\Tn}{\vtheta} + \log\Z(\vtheta)$ are the expected negative log-likelihood for the original data and the perturbed data respectively.

\begin{theorem}\label{rate:MLE}
    The model above fulfills Assumption~\ref{asm:closeloss}, and Assumption~\ref{asm:difference} with $\eps'_n=0$.
	Assume that $\forall j$, $\tf_j(\x)$ follows a sub-Gaussian distribution with parameter $\sigma_{\x}$.
	Suppose the conditional distribution $\vpsi_j(\x,\veta)$ for any fixed $\x$ is sub-Gaussian with parameter $\sigma_{\veta}$, then $\vpsi_j(\x,\veta)$ follows a sub-Gaussian distribution with parameter $\sigma$ such that $\sigma^2 = \sigma_{\x}^2 + \sigma_{\veta}^2$. 
	Thus, we can obtain a rate $\eps_{n,\delta} \in \O(\sigma\sqrt{\sfrac{1}{n}\log{\sfrac{1}{\delta}}})$ for $n$ independent samples.
	
	Similarly, assume that $\forall j$, $\tf_j(\x)$ has variance at most $\sigma_{\x}^2$.
	Suppose the conditional distribution of $\vpsi_j(\x,\veta)$ for any fixed $\x$, has variance at most $\sigma_{\veta}^2$, then $\vpsi_j(\x,\veta)$ has variance at most $\sigma$ such that $\sigma^2 = \sigma_{\x}^2 + \sigma_{\veta}^2$.
	Thus, we can obtain a rate $\eps_{n,\delta} \in \O(\sigma\sqrt{\sfrac{1}{(n \delta)}})$
\end{theorem}
For example, if one uses the $\ell_1$ regularizer \cite{Ravikumar08}, the rate is $\epsilon_{n,\delta} = \sigma \sqrt{\sfrac{2}{n}(\log{p} + \log\sfrac{2}{\delta})}$ for the sub-Gaussian case, and $\epsilon_{n,\delta} = \sigma\sqrt{\frac{p}{n\delta}}$ for the bounded-variance case.
As comparison, the rates with original data \cite{honorio2014unified} are $\sigma_{\x} \sqrt{\sfrac{2}{n}(\log{p} + \log\sfrac{2}{\delta})}$ and $\sigma_{\x}\sqrt{\frac{p}{n\delta}}$ respectively.

\paragraph{Data Irrecoverability.} %\label{privacy:MLE}
Next we provide an example to show how perturbation can prevent an adversarial from recovering the original data. 
Based on the example, we analyze what is the minimum noise to guarantee data irrecoverability.
In what follows, we consider recovering the data up to permutation,  since the ordering of i.i.d. samples in a dataset is not relevant.

Consider a simple example, MLE for an Ising model with zero mean. 
Let $\vtheta \in \BTheta = \R^p$ and $\x\si{i} \in \{-1,1\}^{\sqrt{p}}$ be samples drawn from some unknown distribution. 
Denote $\X = \{\x\si{1},\x\si{2},\ldots,\x\si{n}\}$.
The sufficient statistic is $\tf(\x\si{i}) = \x\si{i}\x\sit{i}$, and the empirical sufficient statistic is $\Th = \sum_{i}\x\si{i}\x\sit{i}$.
We add noise in the following way: we sample $n$ times from $ \N(0,\sigma_\eta^2\mathbf{I})$. 
We then get $\veta\si{i}, i=1,\ldots,n$,
then add noise to samples, obtaining $\X_\eta = \{\x\si{1}+\veta\si{1},\ldots,\x\si{n}+\veta\si{n}\}$. 
The perturbed sufficient statistics becomes $\Thn' = \sum_i(\x\si{i}+\veta\si{i})(\x\si{i}+\veta\si{i})^T$. 
Finally we publish $\Thn$ which we obtain by removing the diagonal entries of $\Thn'$ and by clamping the non-diagonal entries of $\Thn'$ to the range $[-1, 1]$.

\begin{theorem} \label{thm:MLE_pri}
	If we perturb $\Th$ as mentioned above, 
	$\gamma \leq 1 - \frac{4}{n \sqrt{p}}$, $ n \leq 2^{\sqrt{p}/4}$
	and the noise variance fulfills 
    $\sigma_\eta^2 \geq \frac{4}{(1-\gamma)\log 2}$,
	then any adversary will fail to recover the original data up to permutation with probability greater than $\gamma$. That is,
	$
	    \inf_\A \P_{\X,\eta} [\A(\X_\eta) \neq \X] \geq \gamma.
	$
\end{theorem}

Let $\X,\X' \in \XS = \{-1,+1\}^{n \times \sqrt{p}}$ be two datasets and let $d(\X, \X')$ be the number of different samples between $\X$ and $\X'$. 
The maximum neighborhood size at radius $t$ is defined as:
\begin{align*}
N_{\max}(t) = \max_{\X\in \XS}\sum_{\X'\in \XS}\mathbf{1}[d(\X, \X')\leq t]
 = \binom{n}{t} \binom{2^{\sqrt{p}}-n}{t}
\end{align*}
We now state our theorem.
\begin{theorem}\label{thm:MLE_priext}
	Under the same conditions as in Theorem~\ref{thm:MLE_pri}, if $\gamma \leq 1 - \frac{4}{n \sqrt{p}}$, $n \leq 2^{\sqrt{p}/4}$
	and the noise variance fulfills 
    \begin{align*}
    \sigma_\eta^2 \geq \frac{4}{(1\hns-\hns\gamma)\log 2 + (1\hns-\hns\gamma)^2\frac{t}{4}(\log\frac{t^2}{2^{3\sqrt{p}/4}-1}\hns-\hns\frac{\sqrt{p}}{2}\log 2 \hns-\hns 2)}
    \end{align*}
	then any adversary will fail to recover the original data up to permutation with probability greater than $\gamma$. That is,
	$
	    \inf_\A \P_{\X,\eta} [d(\A(\X_\eta), \X) > t] \geq \gamma.
	$
\end{theorem}

\subsection{Generalized Linear Models with Fixed Design} \label{subsec:glm}
Generalized linear models unify different models, including linear regression (when Gaussian noise is assumed), logistic regression and compressed sensing with exponential-family noise \cite{Rish09}.
For simplicity, we focus on the fixed design model, in which $y$ is an random variable and $\x$ is a constant vector. 
Let $t(y)$ be the sufficient statistic and $\Z(\nu) = \int_y{\exp{t(y)\nu}}$ be the partition function. 
Let $\Lh(\vtheta) = \frac{1}{n}\sum_{i}-t(y\si{i})\dotprod{\x\si{i}}{\vtheta} + \log\Z(\dotprod{\x\si{i}}{\vtheta})$ be the empirical negative log-likelihood for original data $y\si{i}$ given their linear predictions $\dotprod{\x\si{i}}{\vtheta}$.
Let $\Lhn(\vtheta) = \frac{1}{n}\sum_{i}-\psi(y\si{i},\eta\si{i}) \dotprod{\x\si{i}}{\vtheta} + \log\Z(\dotprod{\x\si{i}}{\vtheta})$ be the empirical negative log-likelihood for privatized data $y\si{i}$ given their linear predictions $\dotprod{\x\si{i}}{\vtheta}$.
Similarly, $\L(\vtheta) = \E_{(\forall i) y\si{i} \sim \DD_i}[\Lh(\vtheta)]$ and $\Ln(\vtheta) = \E_{(\forall i) y\si{i} \sim \DD_i,\eta\si{i} \sim \ND}[\Lhn(\vtheta)]$ are the expected negative log-likelihood for the original and the perturbed data respectively.

\begin{theorem}\label{rate:GLM}
    The model above fulfills Assumption~\ref{asm:closeloss}, and Assumption~\ref{asm:difference} with $\eps'_n=0$.
	Assume that $t(y)$ follows a sub-Gaussian distribution with parameter $\sigma_y$.
	Suppose the conditional distribution of $\psi(y,\eta)$ for any fix $y$ is sub-Gaussian with parameter $\sigma_\eta$, then $\psi(y,\eta)$ follows a sub-Gaussian distribution with parameter $\sigma$, such that $\sigma^2 = \sigma_y^2 + \sigma_\eta^2$. 
	Thus, we can obtain a rate $\varepsilon_{n,\delta} \in \O(\sigma\sqrt{\sfrac{1}{n}\log{\sfrac{1}{\delta}}})$.
	
	Similarly, assume that $t(y)$ has variance at most $\sigma_y^2$, and that the conditional distribution of $\psi(y,\eta)$ for any fixed $y$ has variance at most $\sigma_\eta^2$, then $\psi(y,\eta)$ has  variance at most $\sigma^2$ with $\sigma^2 = \sigma_y^2 + \sigma_\eta^2$.
	Thus, we can obtain a rate $\eps_{n,\delta} \in \O(\sigma\sqrt{\sfrac{1}{(n \delta)}})$
\end{theorem}

As comparison, the rates with original data \cite{honorio2014unified} are $\O(\sigma_{y}\sqrt{\frac{1}{n}\log \frac{2}{\delta}})$ and $\O(\sigma_y \sqrt{\frac{1}{n\delta}})$ respectively.

\paragraph{Data Irrecoverability.} %\label{privacy:GLM}
Next we provide an example and show the minimum noise to achieve data irrecoverability. 
Here, we only consider to protect $y$.
Assume that $y\si{i} \in \{+1,-1\}$ is drawn from some unknown data distribution. 
Let the sufficient statistic $t(y) = y$.
Denote $Y=\{y\si{1},\ldots,y\si{n}\}$.
We sample $n$ times from $\N(0,\sigma_\eta^2)$, and get $\eta\si{1}, \ldots, \eta\si{n}$.
Then we perturb the data as $\psi(y, \eta) = y + \eta$.
Finally we publish $Y_\eta = \{y\si{1}+\eta\si{1}, \ldots, y\si{n}+\eta\si{n}\}$ and all corresponding $\x\si{i}$.

\begin{theorem} \label{thm:GLM_pri}
	If we perturb $Y$ as mentioned above, $\gamma \leq 1- \frac{2}{n}$ and the noise variance fulfills $\sigma_\eta^2 \geq  \frac{2}{(1-\gamma)\log 2}$, then any adversary will fail to recover the original data with probability greater than $\gamma$. 
	That is,
	$
	    \inf_\A \P_{Y,\eta} [\A(Y_\eta) \neq Y] \geq \gamma.
	$
\end{theorem}

Let $Y,Y' \in \YS = \{-1,+1\}^n$ be two datasets and let $d(Y, Y')$ be the number of different samples between $Y$ and $Y'$. 
The maximum neighborhood size at radius $t$ is defined as:
\begin{align*}
N_{\max}(t) = \max_{Y\in \YS}\sum_{Y'\in \YS}\mathbf{1}[d(Y, Y')\leq t] = \binom{n}{t}
\end{align*}
We now state our theorem.
\begin{theorem} \label{thm:GLM_priext}
	Under the same conditions as in Theorem~\ref{thm:GLM_pri}, if $\gamma \leq 1- \frac{2}{n}$ and the noise variance fulfills
	\begin{align*}
	\sigma_\eta^2 &\geq \frac{2}{(1-\gamma) \log 2+(1-\gamma)^2 t (\log \frac{t(1-\gamma)}{2} - 1 )}
	\end{align*}
	then any adversary will fail to recover the original data with probability greater than $\gamma$. 
	That is,
	$
	    \inf_\A \P_{Y,\eta} [d(\A(Y_\eta), Y) > t] \geq \gamma.
	$
\end{theorem}

\subsection{Exponential-family PCA} \label{subsec:pca}
Exponential family PCA was first introduced in \cite{Collins01} as a generalization of Gaussian PCA.
We assume that each entry in in the random matrix $\X \in \R^{n_1\times n_2}$ is independent, and might follow a different distribution. 
The hypothesis space for this problem is $\vtheta \in \BTheta = \R^{n_1\times n_2}$. 
Let $t(x_{ij})$ be the sufficient statistic and and $\Z(\nu) = \int_{x_{ij}}\exp{t(x_{ij})\nu}$ be the partition function.
Let $\Lh(\vtheta) = \frac{1}{n}\sum_{ij}{-t(x_{ij})\theta_{ij} + \log\Z(\theta_{ij})}$ be the empirical negative log-likelihood for original data $x_{ij}$.
Let $\Lhn(\vtheta) = \frac{1}{n}\sum_{ij}{-\psi(x_{ij},\eta_{ij})\theta_{ij} + \log\Z(\theta_{ij})}$ be the empirical negative log-likelihood for privatized data $\psi(x_{ij},\eta_{ij})$.
Denote $\L(\vtheta) = \E_{(\forall ij){\rm\ }x_{ij} \sim \DD_{ij}}[\Lh(\vtheta)]$ and $\Ln(\vtheta) = \E_{(\forall ij){\rm\ }x_{ij} \sim \DD_{ij}, \eta_{ij} \sim \ND}[\Lhn(\vtheta)]$ as the expected negative log-likelihood function for the original and the perturbed data.

\begin{theorem}\label{rate:PCA}
    The model above fulfills Assumption~\ref{asm:closeloss}, and Assumption~\ref{asm:difference} with $\eps'_n=0$.
	Assume that $t(x_{ij})$ follows a sub-Gaussian distribution with parameter $\sigma_{\x}$.
	Suppose the conditional distribution of $\psi(x_{ij},\eta_{ij})$ for any fix $x_{ij}$ is sub-Gaussian with parameter $\sigma_\eta$, then $\psi(x_{ij},\eta_{ij})$ follows a sub-Gaussian distribution with parameter $\sigma$, such that $\sigma^2 = \sigma_\x^2 + \sigma_\eta^2$. 
	Thus, we can obtain a rate $\eps_{n,\delta} \in \O(\sigma\sqrt{\sfrac{1}{n}\log{\sfrac{1}{\delta}}})$.
	
	Similarly, assume that $t(x_{ij})$ has variance at most $\sigma_{\x}^2$, and that the conditional distribution of $\psi(x_{ij},\eta_{ij})$ for any fixed $x$ has variance at most $\sigma_\eta^2$, then $\psi(x_{ij},\eta_{ij})$ has variance at most  $\sigma^2$ such that $\sigma^2 = \sigma_\x^2 + \sigma_\eta^2$.
	Thus, we can obtain a rate $\eps_{n,\delta} \in \O(\sigma\sqrt{\sfrac{1}{(n \delta)}})$.
\end{theorem}
As comparison, the rates with original data \cite{honorio2014unified} are $\O(\sigma_\x\sqrt{\sfrac{1}{n}\log{\sfrac{1}{\delta}}})$ and $\O(\sigma_\x\sqrt{\sfrac{1}{(n \delta)}})$ respectively.

\paragraph{Data Irrecoverability.} %\label{privacy:PCA}
Next we provide an example and show the minimum noise to achieve data irrecoverability.
Assume $\forall ij,  x_{ij} \in \{-1,+1\} $.
We perturb the data in the way that $\psi(x_{ij},\eta_{ij}) = x_{ij} + \eta_{ij}$, where $\eta_{ij} \sim \N(0,\sigma_\eta^2)$.
Let $\X$ denote the original data, $\X_\eta$ denote the perturbed data. 
That is, the $(i,j)$-th entry of $\X_\eta$ is $\psi(x_{ij},\eta_{ij})$.
\begin{theorem} \label{thm:PCA_pri}
	If we perturb $\X$ as mentioned above, $\gamma\leq 1- \frac{2}{n}$ and the noise variance fulfills $\sigma_\eta^2 \geq \frac{2}{(1-\gamma)\log 2}$, then any adversary will fail to recover the original data with probability greater than $\gamma$.
	That is,
	$
	    \inf_\A \P_{\X,\eta} [\A(\X_\eta) \neq \X] \geq \gamma.
	$
\end{theorem}

Let $\X,\X' \in \XS = \{-1,+1\}^{n_1 \times n_2}$ be two matrices and let $d(\X, \X')$ be the number of different entries between $\X$ and $\X'$. 
The maximum neighborhood size at radius $t$ is defined as:
\begin{align*}
N_{\max}(t) = \max_{\X\in \XS}\sum_{\X'\in \XS}\mathbf{1}[d(\X, \X')\leq t] = \binom{n}{t}
\end{align*}
We now state our theorem.
\begin{theorem} \label{thm:PCA_priext}
	Under the same conditions as in Theorem~\ref{thm:PCA_pri}, if $\gamma\leq 1- \frac{2}{n}$ and the noise variance fulfills
	\begin{align*}
	\sigma_\eta^2 &\geq \frac{2}{(1-\gamma) \log 2+(1-\gamma)^2 t (\log \frac{t(1-\gamma)}{2} - 1 )}
	\end{align*}
	then any adversary will fail to recover the original data with probability greater than $\gamma$.
	That is,
	$
	    \inf_\A \P_{\X,\eta} [d(\A(\X_\eta), \X) > t] \geq \gamma.
	$
\end{theorem}

\subsection{Nonparametric Generalized Regression with Fixed Design }\label{subsec:np}
In nonparametric generalized regression with exponential-family noise, the goal is to learn a function, which can be represented in an infinite dimensional orthonormal basis. 
One instance of this problem is the Gaussian case provided in \cite{Ravikumar05} with orthonormal basis functions depending on single coordinates. 
Here we allow for the number of basis functions to grow with more samples. 
For simplicity, we analyze the fixed design model, i.e., $y$ is a random variable and $\x$ is a constant.

Let $\XS$ be the domain of x. 
Let $\theta : \XS \rightarrow \R$ be a predictor.
Let $t(y)$ be the sufficient statistic and $\Z(\nu) = \int_{y} \exp{t(y)\nu}$ be the partition function.
Let $\Lh(\theta) = \frac{1}{n}\sum_{i}{-t(y\si{i})\theta(\x\si{i}) + \log\Z(\theta(\x\si{i}))}$ be the empirical negative log-likelihood for original data $y\si{i}$ given their predictions $\theta(\x\si{i})$.
Let $\Lhn(\theta) = \frac{1}{n}\sum_{i}{-\psi(y\si{i},\eta\si{i})\theta(\x\si{i}) + \log\Z(\theta(\x\si{i}))}$ be the empirical negative log-likelihood for privatized data $\psi(y\si{i},\eta\si{i})$ given their predictions $\theta(\x\si{i})$.
Then denote  $\L(\theta) = \E_{(\forall i){\rm\ }y\si{i} \sim \DD_{i}}[\Lh(\theta)]$ and  $\Ln(\theta) = \E_{(\forall i){\rm\ }y\si{i} \sim \DD_{i}, \eta\si{i} \sim \ND}[\Lhn(\theta)]$ as the expected negative log-likelihood function for the original and the perturbed data.

\begin{theorem}\label{rate:np}
    The model above fulfills Assumption~\ref{asm:closeloss}, and Assumption~\ref{asm:difference} with $\eps'_n=0$.
	Assume that $t(y)$ follows a sub-Gaussian distribution with parameter $\sigma_y$.
	Suppose the conditional distribution of $\psi(y,\eta)$ for any fix $y$ is sub-Gaussian with parameter $\sigma_\eta$, then $\psi(y,\eta)$ follows a sub-Gaussian distribution with parameter $\sigma$, such that $\sigma^2 = \sigma_y^2 + \sigma_\eta^2$. 
	Thus, we can obtain a rate $\eps_{n,\delta} \in \O(\sigma(\sfrac{1}{n^{1/2-\beta}})\sqrt{\log{\sfrac{1}{\delta}}})$ with $n$ independent samples and $\O(\exp{n^{2\beta}})$ basis functions, where $\beta \in (0,1/2)$.
	
	Similarly, assume that $t(y)$ has variance at most $\sigma_y^2$, and that the conditional distribution $\psi(y,\eta)$ for any fixed $y$ has variance at most $\sigma_\eta^2$, then $\psi(y,\eta)$ has variance at most $\sigma^2$ such that $\sigma^2 = \sigma_y^2 + \sigma_\eta^2$.
	Thus, we can obtain a rate $\eps_{n,\delta} \in \O(\sigma(\sfrac{1}{n^{1/2-\beta}})\sqrt{\sfrac{1}{\delta}})$ for $n$ independent samples and $O( n^{2\beta})$ basis functions,  where $\beta \in (0,1/2)$.
	
\end{theorem}
As comparison, the rates with original data \cite{honorio2014unified} are $\O(\sigma_y(\sfrac{1}{n^{1/2-\beta}})\sqrt{\log{\sfrac{1}{\delta}}})$ and $\O(\sigma_y(\sfrac{1}{n^{1/2-\beta}})\sqrt{\sfrac{1}{\delta}})$ respectively.

\paragraph{Data Irrecoverability.}
In the case of nonparametric generalized regression with fixed design, we can perturb the data $y$ in the same way as for generalized linear models with fixed design.
Therefore, Theorems~\ref{thm:GLM_pri} and \ref{thm:GLM_priext} also holds for the nonparametric generalized regression.

\subsection{Max-margin Matrix Factorization} \label{subsec:max-margin}
The max-margin matrix factorization problem was introduced in \cite{Srebro04}, which used a hing loss.
Here we generalize the loss function to Lipschitz continuous.
Let $f: \R \rightarrow \R$ be a $K$ Lipschitz continuous loss function. 
Assume the entries of the random matrix $\X \in \{-1,+1\}^{n_1\times n_2}$ are independent.
Let $n=n_1 n_2$. 
We perturb each of the entries in matrix $\X$ as $\psi(x_{ij}, \eta_{ij}) = x_{ij} \eta_{ij}$, where $P[\eta_{ij}=1] = q$ and $P[\eta_{ij}=-1] = 1-q$. 
Let $\Lh(\vtheta) = \frac{1}{n}\sum_{ij}f(x_{ij}\theta_{ij})$ be the empirical risk of predicting the binary value $x_{ij}\in \{-1, +1\}$ by using ${\rm sgn}(\theta_{ij})$.
Let $ \Lhn(\vtheta) = \frac{1}{n}\sum_{ij}f(\psi(x_{ij}, \eta_{ij})\theta_{ij})$ be the empirical risk of predicting the privatized data $\psi(x_{ij},\eta_{ij})$  by using ${\rm sgn}(\theta_{ij})$.

\begin{theorem}\label{rate:max-margin}
   The model above fulfills Assumption~\ref{asm:closeloss} with probability 1(i.e., $\delta=0$), scale function $c(\vtheta) = \|\vtheta\|_1$ and rate $\eps_{n,0}=\O(1/n)$. 
   The model also fulfills Assumption~\ref{asm:difference} with $\eps'_n \in \O(\frac{K(1-q)}{n})$ and scale function $c(\vtheta) = \|\vtheta\|_1$. 
\end{theorem}
As comparison, the rate with original data \cite{honorio2014unified} is $\O(1/n)$.

\paragraph{Data Irrecoverability.}
We show that data irrecoverability can be achieved in this model.
Let $\X$ denote the original data, $\X_\eta$ denote the perturbed data. 
That is, the $(i, j)$-th entry of $\X_\eta$ is $\psi(x_{ij}, \eta_{ij}) = x_{ij} \eta_{ij}$, where $P[\eta_{ij}=1] = q$ and $P[\eta_{ij}=-1] = 1-q$.
\begin{theorem} \label{thm:max-margin_pri}
    If we perturb $\X$ as mentioned above, $\gamma \leq 1- \frac{2}{n}$ and $q \in (1/2, 1/2 + \frac{(1-\gamma) \log 2}{8} )$, then any adversary will fail to recover the original data with probability greater than $\gamma$. 
	That is,
	$
	    \inf_\A \P_{\X,\veta} [\A(\X_\eta) \neq \X] \geq \gamma.
	$
\end{theorem}

Let $\X,\X' \in \XS = \{-1,+1\}^{n_1 \times n_2}$ be two matrices and let $d(\X, \X')$ be the number of different entries between $\X$ and $\X'$. 
The maximum neighborhood size at radius $t$ is defined as:
\begin{align*}
N_{\max}(t) = \max_{\X\in \XS}\sum_{\X'\in \XS}\mathbf{1}[d(\X, \X')\leq t] = \binom{n}{t}
\end{align*}
We now state our theorem.
\begin{theorem}\label{thm:max-margin_priext}
    Under the same conditions as in Theorem~\ref{thm:max-margin_pri}, if $\gamma \leq 1- \frac{2}{n}$, $G_{\gamma, n, t}=(1-\gamma) \left(\log 2+\frac{t}{n}(\log \frac{t}{n} - 1 ) \right) -\frac{\log 2}{n}$ and
    \begin{align*}
    q\in (\frac{1}{2}, \frac{1}{2}+\frac{-G_{\gamma, n, t}+\sqrt{G_{\gamma, n, t}(G_{\gamma, n, t} + 8)}}{8})
    \end{align*}
    then any adversary will fail to recover the original data with probability greater than $\gamma$. 
	That is,
	$
	    \inf_\A \P_{\X,\veta} [d(\A(\X_\eta), \X) > t] \geq \gamma.
	$
\end{theorem}

\section{Concluding Remarks} \label{sec:future_work}
As a corollary of our result on perturbed loss consistency, we believe that norm consistency, sparsistency and sign consistency as in \cite{honorio2014unified} can also be proved under our framework of data irrecoverability. 
In addition, there are several problems that our current framework cannot accommodate, such as nonparametric clustering with exponential families, for instance.
We need to explore new mathematical characterizations in the context of these problems.

\bibliographystyle{abbrv}
\bibliography{references}

\clearpage
\appendix
\onecolumn

\section{Detailed Proofs} \label{app:proofs}

\subsection{Proof of Theorem~\ref{thm:irrecov}}
\begin{proof}
We invoke Definition~\ref{def:dp} for sets $\mathcal{S}$ of size 1.
In this case we have $\mathcal{S}=\{z\}$ for $z \in \mathcal{Z}$, and therefore $\mathcal{M}(x) \in \mathcal{S}$ is equivalent to $\mathcal{M}(x) = z$.
Furthermore, we use arbitrary datasets $x$ and $x'$.

We can describe the data process with the Markov chain $X \rightarrow \M(X) \rightarrow \A(\M(X))$.
Next, for a fixed and arbitrary $x' \in \XS$, we define the distribution $\mathbb{Q}$ as follows: 
\begin{align*}
    \mathbb{Q}(z) = \frac{e^\epsilon \P_{\M}(\M(x')=z)+\delta}{\int_{z'\in \Z} (e^\epsilon \P_{\M}(\M(x')=z')+\delta) dz'}
\end{align*}
The denominator is a partition function.
It is easy to see that $\mathbb{Q}$ is a valid distribution since $\int_{z \in \mathcal{Z}} \mathbb{Q}(z) dz = 1$.
Then we can bound the mutual information between $X$ and $\M(X)$ in the following way:
\begin{align*}
    \quad \MI(X;\M(X))
    &\leq \frac{1}{|\XS|} \sum_{x \in \XS}  \KL(\P_{\M}(\M(x)) | \mathbb{Q}) \\
    & = \frac{1}{|\XS|} \sum_{x \in \XS} \int_{z \in \Z} \P_{\M}(\M(x)=z) \log (\frac{\P_{\M}(\M(x)=z)}{\mathbb{Q}(z)})dz \\
    & \leq  \frac{1}{|\XS|} \sum_{x \in \XS} \int_{z \in \Z} \P_{\M}(\M(x)=z)
    \log (\frac{e^\epsilon \P_\M(\M(x') = z) + \delta}{\mathbb{Q}(z)})dz \\
    & =  \frac{1}{|\XS|} \sum_{x \in \XS} \int_{z \in \Z} \P_{\M}(\M(x)=z)dz
    \log (\int_{z'\in \Z} (e^\epsilon \P_{\M}(\M(x')=z')+\delta )dz')\\
    & =  \log \int_{z'\in \Z} (e^\epsilon \P_{\M}(\M(x')=z')+\delta )dz'
\end{align*}
The first inequality comes from equation 5.1.4 in \cite{duchifano}.
The second inequality comes from the Definition~\ref{def:dp}.
Since $x'$ is an arbitrary choice in our argument, we can take the infimum with respect to $x'$ and get a tight bound on the mutual information:
\begin{align} \label{eq:irrecov_MI}
    \MI(X;\M(X)) &\leq \inf_{x' \in \XS}  \log \int_{z\in \Z} (e^\epsilon \P_{\M}(\M(x')=z)+\delta )dz \nonumber\\
    & = b(\epsilon,\delta)
\end{align}
Then, by  Fano's inequality \cite{cover2012elements}, we have:
\begin{align*} 
    \inf_\A \P_{X,\M}[\A(\M(X)) \neq X]
    &\geq 1 - \frac{\MI(X;\M(X)) + \log 2}{H(X)}\\
    &\geq 1 - \frac{b(\epsilon,\delta) + \log 2}{H(X)},
\end{align*} 
and we prove our claim.
\end{proof}

\subsection{Proof of Theorem~\ref{thm:irrecovext}}
\begin{proof}
We proceed as in Theorem~\ref{thm:irrecov}, except for the Fano's inequality step.
We now use the Fano's inequality from \cite{duchi2013distance} together with eq.\eqref{eq:irrecov_MI} and obtain:
\begin{align*} 
    \inf_\A \P_{X,\M}[d(\A(\M(X)), X) > t]
    &\geq 1 - \frac{\MI(X;\M(X)) + \log 2}{\log\frac{|\XS|}{N_{\max}(t)}}\\
    &\geq 1 - \frac{b(\epsilon,\delta) + \log 2}{\log\frac{|\XS|}{N_{\max}(t)}},
\end{align*} 
and we prove our claim.
\end{proof}

\subsection{Proof of Corollary~\ref{cor:zero}}
\begin{proof}
    When $\delta = 0$, since $\int_{z \in \Z} \P_\M(\M(x') = z)dz = 1$ for all $x' \in \XS$, we have:
    \begin{align*}
         b(\epsilon, \delta) &= \log \inf_{x' \in \XS} \int_{z\in \Z} (e^\epsilon \P_{\M}(\M(x')=z))dz \\
                     & = \log e^\epsilon \int_{z\in \Z} \P_{\M}(\M(x')=z)dz\\
                     & = \epsilon
    \end{align*}
    By Theorem~\ref{thm:irrecov}, we prove our claim.
\end{proof}

\subsection{Proof of Theorem~\ref{thm:pri_lossconsistency}}
\begin{proof}

By definition, we have 
\begin{eqnarray} \label{proof_thm2_0}
	\Ln(\vthetan^*) - \Ln(\vtheta^* )	\leq 0,
\end{eqnarray}
because $\vthetan^* = \argmin_{\vtheta \in \BTheta}{\Ln(\vtheta)}$.
By Assumptions~\ref{asm:closeloss} and \ref{asm:superreg}, and by setting $\lambda_n = \alpha \eps_{n,\delta}$ for some $\alpha \geq 2$, then we have
\begin{equation} \label{proof_thm2_1}
    \begin{split}
        \Ln(\vthetahn) - \Ln(\vthetan^*) \leq & \eps_{n,\delta} (-\alpha r(c(\vthetahn)) +  c(\vthetahn)) + \\
        & \eps_{n,\delta}(\alpha \Reg(\vthetan^*) + c(\vthetan^*)) + \xi
    \end{split}
\end{equation}
By Assumption~\ref{asm:difference}, and since $\eps_n' \leq \eps_{n,\delta}$, we have
	\begin{eqnarray*}
	\L(\vthetahn) - \L(\vtheta^*)
	&= &	 (\L(\vthetahn) - \Ln(\vthetahn))+ (\Ln(\vthetahn) - \Ln(\vthetan^*)) +
	(\Ln(\vthetan^*) - \Ln(\vtheta^*)) +(\Ln(\vtheta^*)-\L(\vtheta^*))\\
	&\leq&  \eps'_n c(\vthetahn) + \eps_{n,\delta} (-\alpha r(c(\vthetahn)) +  c(\vthetahn)) +
	{n,\delta}(\alpha \Reg(\vthetan^*) + c(\vthetan^*)) + \xi +0 + \eps'_n c(\vtheta^*)\\
	& \leq & \eps_{n,\delta} (-\alpha r(c(\vthetahn)) + 2 c(\vthetahn)) +
	\eps_{n,\delta}(\alpha \Reg(\vthetan^*) + c(\vthetan^*)) + \xi + \eps'_n c(\vtheta^*)\\
	& \leq & \eps_{n,\delta}(\alpha\Reg(\vthetan^*)+c(\vthetan^*)) + \eps'_n c(\vtheta^*) +\xi\\.
	\end{eqnarray*}
The first inequality is based on Assumption~\ref{asm:difference} and the two inequalities (\ref{proof_thm2_0}) and (\ref{proof_thm2_1}) mentioned above. 
The second inequality comes from $\eps_n' \leq \eps_{n,\delta}$.
The third inequality comes from $\alpha \geq 2$, Assumption~\ref{asm:superreg} and the elimination of the negative terms.
\end{proof}

\subsection{Proof of Lemma~\ref{lemma:subgaussian}}
Since $\tf_j(\x)$ follows a sub-Gaussian distribution, then we have $\E_\x[\exp{\lambda(\tf_j(\x) - \E_\x[\tf_j(\x)])}] \leq \exp{\frac{\sigma_x^2 \lambda^2}{2}}$.
Since the conditional random variable $\vpsi_j(\x,\veta)$ for any fixed $\x$ follows sub-Gaussian distribution, then we have $\E_\veta [\exp{\lambda(\vpsi_j(\x,\veta) - \E_\veta[\vpsi_j(\x,\veta)])}|\x] \leq \exp{\frac{\sigma_\eta^2 \lambda^2}{2}}$. 
Thus, for random variable $\vpsi_j(\x,\veta)$ for any $\x$ and $\veta$, we can get:
\begin{eqnarray*}
	\E_{\x,\veta}[\exp{\lambda(\vpsi_j(\x,\veta) - \E_{\x,\veta}[\vpsi_j(\x,\veta)])}]
    &=& \E_{\x,\veta}[\exp{\lambda(\vpsi_j(\x,\veta) -\tf_j(\x) + \tf_j(\x) - \E_{\x,\veta}[\vpsi_j(\x,\veta)])}] \\
    &=& \E_{\x,\veta}[\exp{\lambda(\vpsi_j(\x,\veta) -\E_\veta[\vpsi_j(x,\eta)] + \tf_j(\x) - \E_{\x}[\tf_j(\x)])}]\\
    &=& \E_{\x}[\exp{\lambda(\tf_j(\x) - \E_{\x}[\tf_j(\x)])} \E_\veta[\exp{\lambda(\vpsi_j(\x,\veta) - \E_\veta[\vpsi_j(\x,\veta)])}| \x]]\\
    &\leq& \E_{\x}[\exp{\lambda(\tf_j(\x) - \E_{\x}[\tf_j(\x)])} \exp{\frac{\sigma_\eta^2 \lambda^2}{2}}]\\
    &=& \exp{\frac{(\sigma_x^2+\sigma_\eta^2) \lambda^2}{2}}
\end{eqnarray*}
Thus, $\vpsi_j(\x,\veta)$ will also be sub-Gaussian with parameter $\sigma$ such that $\sigma^2 = \sigma_x^2+\sigma_\eta^2$.

\subsection{Proof of Lemma~\ref{lemma:finitevariance}}
Since $\tf_j(\x)$ has variance at most $\sigma_x^2$ and $\vpsi_j(\x,\veta)$ for any fixed $\x$ has variance at most $\sigma_\eta^2$.
Then for random variable $\vpsi_j(\x,\veta)$ for $\x$ and $\veta$, we have:
\begin{eqnarray*}
	\E_{\x,\veta}[(\vpsi_j(\x,\veta) - \E_{\x,\veta}[\vpsi_j(\x,\veta)])^2]
    &=& \E_{\x,\veta}[(\vpsi_j(\x,\veta) - \tf_j(\x) + \tf_j(\x) -\E_{\x,\veta}[\vpsi_j(\x,\veta)])^2]\\
    &=& \E_{\x,\veta}[(\vpsi_j(\x,\veta) - \tf_j(\x))^2 + \\
    & & 2(\vpsi_j(\x,\veta) - \tf_j(\x))(\tf_j(\x) -\E_{\x,\veta}[\vpsi_j(\x,\veta)]) \\
    & & + (\tf_j(\x) -\E_{\x,\veta}[\vpsi_j(\x,\veta)])^2] \\
    &=& \E_{\x}[\E_\veta[((\vpsi_j(\x,\veta) - \E_{\veta}[\vpsi_j(\x,\veta)])^2) | \x]] +\\
    &  &  2\E_{\x}[(\tf_j(\x) -\E_{\x}[\tf_j(\x)])E_{\veta}[\vpsi_j(\x,\veta) - \E_{\veta}[\vpsi_j(\x,\veta)]]]\\
    & & + \E_{\x}[(\tf_j(\x) -\E_{\x}[\tf_j(\x)])^2]\\
    & \leq & \sigma_\eta^2 + \sigma_x^2
\end{eqnarray*}
We can have last inequality because $E_{\veta}[\vpsi_j(\x,\veta) - \E_{\veta}[\vpsi_j(\x,\veta)]] = 0$.
Thus,  $\vpsi_j(\x,\veta)$ has variance at most $\sigma_\eta^2 + \sigma_x^2$.

\subsection{Proof of Theorem~\ref{rate:MLE}}
\begin{claim}[]\label{clm:mle0}
    The maximum likelihood estimation for exponential family distribution fulfills Assumption~\ref{asm:closeloss} with probability at least $1-\delta$, scale function $c(\vtheta) = \|\vtheta\|$ and rate $\eps_{n,\delta}$, provided that the dual norm fulfills $\|\Thn-\Tn\|_* \leq \eps_{n,\delta}$.
    
    The problem also fulfills Assumption~\ref{asm:difference} with $\eps'_n =0$.
\end{claim}
\begin{proof}
	First we show that $\Ln(\vtheta) = \L(\vtheta)$ for any $\vtheta$.
	Recall that $\mathbb{E}_{\veta}[\psi(\x,\veta)] = t(\x)$.
	We have
	\begin{eqnarray*}
	\Ln(\vtheta) &=& -\dotprod{\Tn}{\vtheta} + \log\Z(\vtheta) \\
	&=& -\dotprod{\T}{\vtheta} + \log\Z(\vtheta) \\
	&=&	\L(\vtheta)\\
	\end{eqnarray*}
	For proving that Assumption~\ref{asm:difference} holds, note that $\Ln(\vtheta) = \L(\vtheta)$ for any $\vtheta$, and thus $\eps'_n = 0$.
	
	For proving that Assumption~\ref{asm:closeloss} holds, we invoke Claim i in \cite{honorio2014unified}, that is for all $\vtheta$
	\begin{align*}
	     |\Lhn(\vtheta) - \Ln(\vtheta)| &= |\dotprod{\Thn - \Tn}{\vtheta}| \\
	    &\leq \|\Thn - \Tn\|_*\|\vtheta\| \\
	    &\leq \eps_{n,\delta}\|\vtheta\| 
	\end{align*}
\end{proof}

Let $\vtheta \in \BTheta = \R^p$. Let $\|\cdot\|_* = \|\cdot\|_\infty$,  $\|\cdot\|=\|\cdot\|_1$.
According to Lemma~\ref{lemma:subgaussian} and Lemma~\ref{lemma:finitevariance}, the variance of  $\vpsi_j(\x, \veta)$ is $\sigma^2 = \sigma^2_x + \sigma^2_\eta$.
We now focus on proving that $\|\Thn - \Tn\|_* \leq \eps_{n,\delta}$ which is the precondition of Claim~\ref{clm:mle0}.

	\paragraph{Sub-Gaussian case and $\ell_1$-norm.}
	 For sub-Gaussian $\psi_j(\x,\veta), 1\leq j \leq p$ with parameter $\sigma$ and $\mathit{l}_1$-norm, by the union bound and independence:
	\begin{eqnarray*}
    \P[\|\Thn - \Tn\|_* > \eps]
	&=& \P[(\exists j) |\frac{1}{n}\sum_{i}^{}(\psi_j(\x\si{i},\veta\si{i})) - \E_{\x\sim \DD}[t_j(\x)]| > \varepsilon] \\
	& = & \P[(\exists j) |\frac{1}{n}\sum_{i}^{}(\psi_j(\x\si{i},\veta\si{i})) - 
	\E_{\x\sim \DD}[\E_{\veta \sim \ND}[\psi_j(\x,\veta)]]| > \varepsilon] \\
	&\leq& 2p\P[\frac{1}{n}\sum_{i}^{}(\psi_j(\x\si{i},\veta\si{i})) - 
	\E_{\x\sim \DD}[\E_{\veta \sim \ND}[\psi_j(\x,\veta)]] > \varepsilon]\\
	& = & 2p\P[\lexp{t(\sum_{i}^{}(\psi_j(\x\si{i},\veta\si{i})) - 
	n\E_{\x \sim \DD}[\E_{\veta \sim \ND}[\psi_j(\x,\veta)]])} > \lexp{tn\varepsilon}] \\
	&\leq& 2p\E[\lexp{(t(\sum_{i}^{}(\psi_j(\x\si{i},\veta\si{i})) - 
	n\E_{\x \sim \DD}[\E_{\veta \sim \ND}[\psi_j(\x,\veta)]]}]/\lexp{tn\varepsilon}\\
	& = & 2p\prod_{i=1}^{n}\mathbb{E}[\lexp{(t(\psi_j(\x\si{i},\veta\si{i})) - 
	\E_{\x \sim \DD}[\E_{\veta \sim \ND}[\psi_j(\x,\veta)]]}]/\lexp{tn\varepsilon} \\
	& \leq &  2p\ \lexp{\frac{\sigma^2t^2n}{2}-tn\varepsilon} \\ 
	&\leq& 2p\ \lexp{-\frac{n\varepsilon^2}{2\sigma^2}} = \delta
	\end{eqnarray*}
	By solving for $\eps$, we have $\eps_{n,\delta} = \sigma \sqrt{\sfrac{2}{n}(\log{p} + \log\sfrac{2}{\delta})}$.

	\paragraph{Finite variance case and $\ell_1$-norm.}
	 For $\psi_j(\x,\veta), 1\leq j \leq p$ with finite variance $\sigma^2$ and $\mathit{l}_1$-norm, by union bound and Chebyshev's inequality:
	\begin{eqnarray*}
	\P[\|\Thn - \Tn\|_* > \eps]
	&=& \P[(\exists j) |\frac{1}{n}\sum_{i}^{}(\psi_j(\x\si{i},\veta\si{i})) - \E_{\x\sim \DD}[t_j(\x)]| > \varepsilon] \\
	& = & \P[(\exists j) |\frac{1}{n}\sum_{i}^{}(\psi_j(\x\si{i},\veta\si{i})) -
	\E_{\x\sim \DD}[\E_{\veta \sim \ND}[\psi_j(\x,\veta)]]| > \varepsilon] \\
	&\leq& p\P[|\frac{1}{n}\sum_{i}^{}(\psi_j(\x\si{i},\veta\si{i})) - 
	\E_{\x\sim \DD}[\E_{\veta \sim \ND}[\psi_j(\x,\veta)]] |> \varepsilon]\\
	&\leq& p\frac{\sigma^2}{n\varepsilon^2}
	\end{eqnarray*}
	By solving for $\varepsilon$, we have $\varepsilon_{n,\delta} = \sigma\sqrt{\frac{p}{n\delta}}$.

\subsection{Proof of Theorem~\ref{thm:MLE_pri}}
\begin{proof}
	Using Fano's inequality, we show that it will be impossible to recover the original data $\X$ up to permutation with probability greater than $1/2$. 
	We can describe the data process with the Markov chain $\X \rightarrow \X_\eta \rightarrow \Thn' \rightarrow \Thn \rightarrow \Xh$, where $\Xh = \A(\Thn)$.
	The mutual information of $\X, \X_\eta$ can be bounded by using the pairwise KL divergence bound \cite{yu1997assouad}.
	\begin{align} \label{eq:MLE_pri_MI}
	\MI[\X;\Xh] & \leq\MI[\X;\Thn] \nonumber\\
	& \leq \MI[\X;\X_\eta]\nonumber\\
	& = n\mathbb{I}[\x\si{i},\x_\veta\si{i}]\nonumber\\
	& \leq \frac{n \sqrt{p}}{|\{-1,+1\}|^2}\sum_{x\si{i}_j\in \{-1,+1\}}\sum_{{x'}\si{i}_j\in \{-1,+1\}}\KL(P_{x\si{i}_{\eta j}|x_j\si{i}} | P_{x\si{i}_{\eta j}|{x'}\si{i}_j})\nonumber\\
	& = \frac{n \sqrt{p}}{|\{-1,+1\}|^2}\sum_{x\si{i}_j\in \{-1,+1\}}\sum_{{x'}\si{i}_j\in \{-1,+1\}}\KL(\N(x\si{i}_j,\sigma_\eta^2) | \N({x'}\si{i}_j,\sigma_\eta^2))\nonumber\\
	& = \frac{n \sqrt{p}}{|\{-1,+1\}|^2}\sum_{x\si{i}_j\in \{-1,+1\}}\sum_{{x'}\si{i}_j\in \{-1,+1\}}\frac{(x\si{i}_j-{x'}\si{i}_j)^2}{2\sigma_\eta^2}\nonumber\\
	& \leq \frac{n \sqrt{p}}{|\{-1,+1\}|^2}(|\{-1,+1\}|^2-|\{-1,+1\}|)\frac{2}{\sigma_\eta^2}\nonumber\\
	& \leq \frac{2 n \sqrt{p}}{\sigma_\eta^2}
	\end{align}
	Because we require the recovery of $\X \in \XS$ up to permutation, we have $|\XS| = \binom{2^{\sqrt{p}}}{n} \geq \frac{2^{\sqrt{p}n}}{n^n}$. 
	By Fano's inequality \cite{cover2012elements} and since $H(\X) \leq \log |\XS|$,
	\begin{align*}
	\mathbb{P}[\Xh \neq \X] &\geq 1-\frac{\mathbb{I}[\X;\Thn]+\log 2 }{\log |\XS|} \\
	&\geq 1-\frac{\frac{2n \sqrt{p}}{\sigma_\eta^2} + \log 2}{n \sqrt{p}\log 2-n\log n}
	\end{align*}
	In order to have $\mathbb{P}[\Xh \neq \X] \geq \gamma$, we require
	\begin{align*}
	\frac{\frac{2n \sqrt{p}}{\sigma_\eta^2} + \log 2}{n \sqrt{p}\log 2-n\log n} &\leq 1-\gamma \\
	\frac{2 + \frac{\sigma_\eta^2 \log 2}{n \sqrt{p}}}{\sigma_\eta^2(\log 2 - \frac{\log n}{\sqrt{p}})} &\leq 1- \gamma\\
	\sigma_\eta^2 &\geq \frac{2}{(1-\gamma)(\log 2 - \frac{\log n}{\sqrt{p}}) - \frac{\log 2}{n \sqrt{p}}} 
	\end{align*}
    Thus, if $n \geq \frac{4}{(1-\gamma)\sqrt{p}}$ and $n \leq 2^{\sqrt{p}/4}$,
    \begin{align*}
        \sigma_\eta^2 \geq \frac{4}{(1-\gamma)\log 2}
    \end{align*}
\end{proof}

\subsection{Proof of Theorem~\ref{thm:MLE_priext}}
\begin{proof}
First, recall that $\Xh = \A(\X_\eta)$.
We proceed as in Theorem~\ref{thm:MLE_pri}, except for the Fano's inequality step.
We now use the Fano's inequality from \cite{duchi2013distance} together with eq.\eqref{eq:MLE_pri_MI} and the fact that $N_{\max}(t) = \binom{n}{t} \binom{2^{\sqrt{p}}-n}{t} \leq \left(\frac{ne}{t}\right)^t \left( \frac{(2^{\sqrt{p}}-n)e}{t}\right)^t$.
Thus, 
\begin{align*}
	\P[d(\Xh, \X) > t] &\geq 1- \frac{\MI(\X;\Xh)+\log 2}{\log \left(\frac{|\XS|}{N_{\max}(t)}\right)}\\
	&\geq 1-\frac{\frac{2n \sqrt{p}}{\sigma_\eta^2} + \log 2}{\log \left(\frac{ t^{2t} }{(e^2n(2^{\sqrt{p}}-n))^t}\frac{2^{\sqrt{p}n}}{n^n}\right)}\\
	&= 1-\frac{\frac{2n \sqrt{p}}{\sigma_\eta^2} + \log 2}{n \sqrt{p}\log 2-n\log n+t(\log \frac{t^2}{n(2^{\sqrt{p}}-n)} - 2 )}\\
\end{align*}
Note that $t\leq n$ in our analysis. 
In order to have $\P[d(\Xh, \X) > t] \geq \gamma$, we require 
\begin{align*}
    &\frac{\frac{2n \sqrt{p}}{\sigma_\eta^2} + \log 2}{n \sqrt{p}\log 2-n\log n+t(\log \frac{t^2}{n(2^{\sqrt{p}}-n)} - 2 )} \leq 1 - \gamma \\
    &\sigma_\eta^2 \geq \frac{2n\sqrt{p}}{(1-\gamma) \left(n \sqrt{p}\log 2-n\log n+t(\log \frac{t^2}{n(2^{\sqrt{p}}-n)} - 2 ) \right) -\log 2} \\
    &\sigma_\eta^2 \geq \frac{2}{(1-\gamma) \left(\log 2 - \frac{\log n}{\sqrt{p}}+\frac{t}{n\sqrt{p}}(\log \frac{t^2}{n(2^{\sqrt{p}}-n)} - 2 ) \right) -\frac{\log 2}{n\sqrt{p}}}
\end{align*}
Thus, if $n \geq \max\{\frac{4}{(1-\gamma)\sqrt{p}}, \frac{5}{4}t+\frac{\log2 + 2t(1-\gamma)( 1- \log t)}{(1-\gamma)\sqrt{p}\log2}\}$ and $n \leq 2^{\sqrt{p}/4}$,
\begin{align*}
    \sigma_\eta^2 \geq \frac{4}{(1-\gamma)\log 2 + (1-\gamma)^2\frac{t}{4}(\log\frac{t^2}{2^{3\sqrt{p}/4}-1}-\frac{\sqrt{p}}{2}\log 2 -2)}
\end{align*}
\end{proof}

\subsection{Proof of Theorem~\ref{rate:GLM}}
\begin{claim}[]\label{clm:GLM0}
    The generalized linear models with fixed design fulfills Assumption~\ref{asm:closeloss} with probability at least $1-\delta$, scale function $c(\vtheta) = \|\vtheta\|$ and rate $\eps_{n,\delta}$, provided that the dual norm fulfills $\|\frac{1}{n}\sum_i{(\psi(y\si{i}, \eta\si{i})-\E_{y \sim \DD_i, \eta \sim \ND_i}[\psi(y\si{i}, \eta\si{i})])\x\si{i}}\|_* \leq \eps_{n,\delta}$.
    
	The problem also fulfills Assumption~\ref{asm:difference} with $\eps'_n =0$ .

\end{claim}
\begin{proof}
   We first show that $\Ln(\vtheta) = \L(\vtheta)$ for any $\vtheta$.
	Recall that  $\mathbb{E}_{\veta}[\psi(y,\eta)] = t(y)$.
	We have
	\begin{align*}
	\Ln (\vtheta) & = \E_{(\forall i) y\si{i} \sim \DD_i,\eta\si{i} \sim \ND}[\Lhn(\vtheta)] \\
	& = \E_{(\forall i) y\si{i} \sim \DD_i,\eta\si{i} \sim \ND}[\frac{1}{n}\sum_{i}-\psi(y\si{i},\eta\si{i}) \dotprod{\x\si{i}}{\vtheta} 
	+ \log\Z(\dotprod{\x\si{i}}{\vtheta})] \\
	& = \E_{(\forall i) y\si{i} \sim \DD_i}[\frac{1}{n}\sum_{i}-\E_{\eta\si{i} \sim \ND}\psi(y\si{i},\eta\si{i}) \dotprod{\x\si{i}}{\vtheta} 
	+ \log\Z(\dotprod{\x\si{i}}{\vtheta})] \\
	& = \E_{(\forall i) y\si{i} \sim \DD_i}[\frac{1}{n}\sum_{i}-t(y\si{i}) \dotprod{\x\si{i}}{\vtheta} + 
	+ \log\Z(\dotprod{\x\si{i}}{\vtheta})] \\
	& = \L(\vtheta)
	\end{align*}
	For proving that Assumption~\ref{asm:difference} holds, note that $\Ln(\vtheta) = \L(\vtheta)$ for any $\vtheta$, and thus $\eps'_n = 0$.
	
	For proving that Assumption~\ref{asm:closeloss} holds, we invoke Claim ii in \cite{honorio2014unified}, that is for all $\vtheta$
	\begin{align*}
	    |\Lhn(\vtheta) - \Ln(\vtheta)| 
	    &= |\frac{1}{n}\sum_{i} \psi(y\si{i}, \eta\si{i})\dotprod{\x ^{(i)}}{\vtheta} - 
	    \frac{1}{n}\sum_{i} \E_{\DD, \ND}[\psi(y\si{i}, \eta\si{i})]\dotprod{\x ^{(i)}}{\vtheta}| \\
	    &= |\dotprod{\frac{1}{n}\sum_{i}(\psi(y\si{i}, \eta\si{i})-\E_{\DD, \ND}[\psi(y\si{i}, \eta\si{i})])\x\si{i}}{\vtheta} | \\ 
	    &\leq \|\frac{1}{n}\sum_{i}(\psi(y\si{i}, \eta\si{i})-\E_{\DD, \ND}[\psi(y\si{i}, \eta\si{i})])\x\si{i}\|_*\|\vtheta\| \\
	    &\leq \eps_{n,\delta}\|\vtheta\|
	\end{align*}
\end{proof}

	Let $\vtheta \in \BTheta =\R^p$. 
	Let $\|\cdot\|_* = \|\cdot\|_\infty$ and $\|\cdot\| = \|\cdot\|_1$.
	Let $\forall x, \|x\|_* \leq B$ and thus $\forall i,j, |x\si{i}_j| < B$.
	According to Lemma~\ref{lemma:subgaussian} and Lemma~\ref{lemma:finitevariance}, the variance of  $\psi(y, \eta)$ is $\sigma^2 = \sigma^2_y + \sigma^2_\eta$.
    We now focus on proving that $\|\frac{1}{n}\sum_{i}(\psi(y\si{i}, \eta\si{i})-\E_{\DD, \ND}[\psi(y\si{i}, \eta\si{i})])\x\si{i}\|_* \leq \eps_{n,\delta}$ which is the precondition of Claim~\ref{clm:GLM0}.
    
	\paragraph{Sub-Gaussian case and $\ell_1$-norm.}
	By Claim~\ref{clm:GLM0}, and by the union bound and independence, if we have sub-Gaussian $\psi(y,\eta)$, then
	\begin{align*}
	\P[\|\frac{1}{n}\sum_{i}(\psi(y\si{i}, \eta\si{i})-\E_{\DD, \ND}[\psi(y\si{i}, \eta\si{i})])\x\si{i}\|_* > \eps] 
	& = \P[(\exists j)|\frac{1}{n}\sum_{i}(\psi(y,\eta)-\E_{y \sim \DD}[t(y)])x\si{i}_j| > \varepsilon]\\
	&\hspace{-0.5in} = \P[(\exists j)|\frac{1}{n}\sum_{i}(\psi(y,\eta)-\E_{y \sim \DD}[\E_{\eta\sim \ND}[\psi(y,\eta)]])x\si{i}_j| 
	> \varepsilon] \\
	&\hspace{-0.5in} \leq 2p\ \lexp{-\frac{n\varepsilon^2}{2(\sigma B)^2}} 
	\end{align*}
	Thus, $\varepsilon_{n,\delta} = \sigma B \sqrt{\sfrac{2}{n}(\log{p} + \log\sfrac{2}{\delta})}$
	
	\paragraph{Finite variance case and $\ell_1$-norm.}
	If $\psi(y,\eta)$ has variance at most $\sigma^2$, then by Claim~\ref{clm:GLM0}, and by the union bound and Chebyshev's  inequality, 
	\begin{align*}
	\P[\|\frac{1}{n}\sum_{i}(\psi(y\si{i}, \eta\si{i})-\E_{\DD, \ND}[\psi(y\si{i}, \eta\si{i})])\x\si{i}\|_* > \eps] 
	& = \P[(\exists j)|\frac{1}{n}\sum_{i}(\psi(y,\eta)-\E_{y \sim \DD}[t(y)])x\si{i}_j| > \varepsilon]\\
	&\hspace{-0.5in} = \P[(\exists j)|\frac{1}{n}\sum_{i}(\psi(y,\eta)-\E_{y \sim \DD}[\E_{\eta\sim \ND}[\psi(y,\eta)]])x\si{i}_j|
	> \varepsilon] \\
	&\hspace{-0.5in} \leq p \frac{(\sigma B)^2}{n\eps^2}
	\end{align*}
	By solving for $\eps$, we have $\eps_{n,\delta} = \sigma B \sqrt{\frac{p}{n\delta}}$.

\subsection{Proof of Theorem~\ref{thm:GLM_pri}}

\begin{proof}
	Using Fano's inequality, we show that it will be impossible to recover the original data $Y$ with probability greater than $1/2$.
	We can describe the data process with the Markov chain $Y \rightarrow Y_\eta \rightarrow \Yh$, where $\Yh = \A(Y_\eta)$.
	The mutual information of $Y, Y_\eta$ can be bounded by using the pairwise KL divergence bound \cite{yu1997assouad}.
	\begin{align} \label{eq:GLM_pri_MI}
	\MI[Y;\Yh] &\leq \MI[Y;Y_\eta]\nonumber\\
	& = n \MI[y\si{i};y\si{i}_\eta]\nonumber\\
	& \leq \frac{n}{|\{-1,+1\}|^2} \sum_{y\si{i}\in \{-1,+1\}} \sum_{{y'}\si{i}\in \{-1,+1\}}\KL(P_{y\si{i}_\eta|y\si{i}}|P_{y\si{i}_\eta|{y'}\si{i}})\nonumber\\
	& = \frac{n}{|\{-1,+1\}|^2} \sum_{y\si{i}\in \{-1,+1\}} \sum_{{y'}\si{i}\in \{-1,+1\}}\KL(\N(y\si{i},\sigma_\eta^2)|\N({y'}\si{i},\sigma_\eta^2))\nonumber\\
	& = \frac{n}{|\{-1,+1\}|^2} \sum_{y\si{i}\in \{-1,+1\}} \sum_{{y'}\si{i}\in \{-1,+1\}} \frac{(y\si{i} - {y'}\si{i})^2}{2\sigma_\eta^2}\nonumber\\
	& \leq \frac{n}{|\{-1,+1\}|^2} (|\{-1,+1\}|^2 - |\{-1,+1\}|) \frac{2}{\sigma_\eta^2}\nonumber\\
	& \leq \frac{n}{\sigma_\eta^2}
	\end{align}
	Since $Y \in \YS = \{-1,+1\}^n$ we have $|\YS| = 2^n$.
	By Fano's inequality\cite{cover2012elements} and since $H(Y) \leq \log |\YS|$,
	\begin{align*}
	\P[\Yh \neq Y] &\geq 1- \frac{\MI(Y;\Yh)+\log 2}{\log |\YS|}\\
	&\geq 1-\frac{\frac{n}{\sigma_\eta^2} + \log 2}{n\log 2}\\
	\end{align*}
	In order to have $P[\Yh \neq Y]\geq \gamma$, we require
	\begin{align*}
	    \frac{\frac{n}{\sigma_\eta^2}+\log 2}{n\log 2} &\leq 1-\gamma \\
	    \frac{1}{\sigma_\eta^2 \log 2} + \frac{1}{n} &\leq 1-\gamma\\
	\end{align*}
	Thus, if $n > \frac{2}{1-\gamma}$, we have
	\begin{align*}
	     \sigma_\eta^2 &\geq \frac{1}{(1-\gamma-\frac{1}{n})\log 2} \\
	     \sigma_\eta^2 &\geq \frac{2}{(1-\gamma)\log 2}
	\end{align*}
\end{proof}

\subsection{Proof of Theorem~\ref{thm:GLM_priext}}
\begin{proof}
First, recall that $\Yh = \A(Y_\eta)$.
We proceed as in Theorem~\ref{thm:GLM_pri}, except for the Fano's inequality step.
We now use the Fano's inequality from \cite{duchi2013distance} together with eq.\eqref{eq:GLM_pri_MI} and the fact that $N_{\max}(t) = \binom{n}{t} \leq \left(\frac{ne}{t}\right)^t$.
Thus, 
\begin{align*}
	\P[d(\Yh, Y) > t] &\geq 1- \frac{\MI(Y;\Yh)+\log 2}{\log \left(\frac{|\YS|}{N_{\max}(t)}\right)}\\
	&\geq 1-\frac{\frac{n}{\sigma_\eta^2} + \log 2}{\log \left(\frac{2^n t^t }{(ne)^t}\right)}\\
	&= 1-\frac{\frac{n}{\sigma_\eta^2} + \log 2}{n\log 2+t(\log \frac{t}{n} - 1 )}\\
	\end{align*}
Note that $t\leq n$ in our analysis. 
In order to have $\P[d(\Yh, Y) > t] \geq \gamma$, we require 
\begin{align*}
    \frac{\frac{n}{\sigma_\eta^2} + \log 2}{n\log 2+t(\log \frac{t}{n} - 1 )} &\leq 1 - \gamma \\
    \sigma_\eta^2 &\geq \frac{n}{(1-\gamma) \left(n\log 2+t(\log \frac{t}{n} - 1 ) \right) -\log 2} \\
    \sigma_\eta^2 &\geq \frac{1}{(1-\gamma) \left(\log 2+\frac{t}{n}(\log \frac{t}{n} - 1 ) \right) -\frac{\log 2}{n}}
\end{align*}
Thus, if $n > \frac{2}{1-\gamma}$, we have
\begin{align*}
    \sigma_\eta^2 &\geq \frac{2}{(1-\gamma) \log 2+(1-\gamma)^2 t (\log \frac{t(1-\gamma)}{2} - 1 )}
\end{align*}
\end{proof}

\subsection{Proof of Theorem~\ref{rate:PCA}}
\begin{claim}[]\label{clm:PCA0}
    The exponential family PCA fulfills Assumption~\ref{asm:closeloss} with probability at least $1-\delta$, scale function $c(\vtheta) = \|\vtheta\|$ and rate $\eps_{n,\delta}$, provided that the dual norm fulfills $\|\frac{1}{n}(\psi(x_{11}, \eta_{11})-\E_{ x\sim \DD_{11}, \eta\sim\ND_{11}}[\psi(x, \eta)], \dots, \psi(x_{n_1n_2}, \eta{x_{n_1n_2}})-\E_{x \sim \DD_{n_1n_2}, \eta \sim \ND_{n_1n_2}}[\psi(x, \eta)])\|_* \leq \eps_{n,\delta}$.
    
    The problem also fulfills Assumption~\ref{asm:difference} with $\eps'_n = 0 $.

\end{claim}
\begin{proof}
    We first show that $\Ln(\vtheta) = \L(\vtheta)$ for any $\vtheta$.
	We have
	\begin{align*}
	\Ln (\vtheta) & = \E_{(\forall ij)x_{ij}\sim \DD_{ij}, \eta_{ij} \sim \ND_{ij}}[\Lhn(\vtheta)] \\
	&= \E_{(\forall ij) x_{ij}\sim \DD_{ij}, \eta_{ij} \sim \ND_{ij}}[\frac{1}{n}\sum_{ij}-\psi(x_{ij}, \eta_{ij})\theta_{ij} 
	+ \log \mathcal{Z}(\theta_{ij})]\\
	&= \E_{(\forall ij) x_{ij}\sim \DD_{ij}}[\frac{1}{n}\sum_{ij}-\E_{\eta_{ij} \sim \ND_{ij}}[\psi(x_{ij}, \eta_{ij})]\theta_{ij}  
	+ \log \mathcal{Z}(\theta_{ij})] \\
	& = \E_{(\forall ij) x_{ij}\sim \DD_{ij}}[\frac{1}{n}\sum_{ij}-t(x)\theta_{ij}  + \log \mathcal{Z}(\theta_{ij})] \\
	&= \mathcal{L}(\mathbf{\theta}) 
	\end{align*}
	For proving that Assumption~\ref{asm:difference} holds, note that $\Ln(\vtheta) = \L(\vtheta)$ for any $\vtheta$, and thus $\eps'_n = 0$.
		
	For proving that Assumption~\ref{asm:closeloss} holds, we have for all $\vtheta$
	\begin{align*}
	     |\Lhn(\vtheta) - \Ln(\vtheta)|
	     & =|\frac{1}{n}\sum_{ij} \psi(x_{ij}, \eta_{ij})\theta_{ij} - \frac{1}{n}\sum_{ij}\E_{x\sim \DD_{ij}, \eta \sim \ND_{ij}}[\psi(x, \eta)]\theta_{ij} |\\
	     &= |\frac{1}{n}\sum_{ij} (\psi(x_{ij}, \eta_{ij}) -\E_{x\sim \DD_{ij}, \eta \sim \ND_{ij}}[\psi(x, \eta)]) \theta_{ij} |\\
	     &\leq \|\frac{1}{n}(\psi(x_{11}, \eta_{11})-\E_{ x\sim \DD_{11}, \eta\sim\ND_{11}}[\psi(x, \eta)], \dots,\\
	     &\quad \quad \psi(x_{n_1n_2}, \eta{x_{n_1n_2}})-\E_{x \sim \DD_{n_1n_2}, \eta \sim \ND_{n_1n_2}}[\psi(x, \eta)])\|_*\|\vtheta\| \\
	     &\leq \eps_{n, \delta}\|\vtheta\|
	\end{align*}
\end{proof}

	Recall that $\vtheta \in \BTheta = \R^{n_1\times n_2}$ and $n = n_1 \times n_2$. 
	Let $\|\cdot\|_* = \|\cdot \|_\infty$, $\|\cdot\| = \|\cdot\|_1$. 
	According to Lemma~\ref{lemma:subgaussian} and Lemma~\ref{lemma:finitevariance}, the variance of  $\psi(x_{ij}, \eta_{ij})$ is $\sigma^2 = \sigma^2_{xij} + \sigma^2_{\eta ij}$.
    We now focus on proving that $\|\frac{1}{n}(\psi(x_{11}, \eta_{11})-\E_{ x\sim \DD_{11}, \eta\sim\ND_{11}}[\psi(x, \eta)],\ \dots,\ \psi(x_{n_1n_2}, \eta{x_{n_1n_2}})-\E_{x \sim \DD_{n_1n_2}, \eta \sim \ND_{n_1n_2}}[\psi(x, \eta)])\|_* \leq \varepsilon_{n,\delta}$ which is the precondition of Claim iii.
	
	\paragraph{Claim~\ref{clm:PCA0}  Sub-Gaussian case and $\ell_1$-norm.}
	If we have sub-Gaussian $\psi(x_{ij}, \eta_{ij})$, by Claim~\ref{clm:PCA0}, and by the union bound and independence, we have
	\begin{align*}
	& \P[\|\frac{1}{n}(\psi(x_{11}, \eta_{11})-\E_{ x\sim \DD_{11}, \eta\sim\ND_{11}}[\psi(x, \eta)],\ \dots,\\
	& \quad \quad \psi(x_{n_1n_2}, \eta{x_{n_1n_2}})-\E_{x \sim \DD_{n_1n_2}, \eta \sim \ND_{n_1n_2}}[\psi(x, \eta)])\|_* >\eps] \\
	&= \P[(\exists ij)|\psi(x_{ij},\eta_{ij}) - \E_{x\sim \DD_{ij}}[t(x_{ij})]|> n\varepsilon] \\
	&\leq 2n\lexp{-\frac{(n\varepsilon)^2}{2\sigma^2}}
	\end{align*}
	Let $\delta = 2n\lexp{-\frac{(n\varepsilon)^2}{2\sigma^2}}$, we still have $\varepsilon_{n,\delta} = \frac{n}{\sigma}\sqrt{2(\log n + \log\frac{2}{\sigma})}$
	
	\paragraph{Claim~\ref{clm:PCA0}  Finite variance case and $\ell_1$-norm.}
	If $\psi(x_{ij}, \eta_{ij})$ has variance at most $\sigma$, by Claim~\ref{clm:PCA0}, and by the union bound and Chebyshev's inequality:
	\begin{align*}
	& \P[\|\frac{1}{n}(\psi(x_{11}, \eta_{11})-\E_{ x\sim \DD_{11}, \eta\sim\ND_{11}}[\psi(x, \eta)],\ \dots,\\
	& \quad \quad \psi(x_{n_1n_2}, \eta{x_{n_1n_2}})-\E_{x \sim \DD_{n_1n_2}, \eta \sim \ND_{n_1n_2}}[\psi(x, \eta)])\|_* >\eps] \\
	&= \P[(\exists ij)|\psi(x_{ij},\eta_{ij}) - \E_{x\sim \DD_{ij}}[t(x)]| > n\varepsilon] \\
	&\leq n\frac{\sigma^2}{(n\varepsilon)^2}
	\end{align*}
	Let $\delta = n\frac{\sigma^2}{(n\varepsilon)^2}$, then we have $\varepsilon_{n,\delta} = \frac{\sigma}{\sqrt{n\sigma}}$

\subsection{Proof of Theorem~\ref{thm:PCA_pri}}
\begin{proof}
Using Fano's inequality, we show that it will be impossible to recover the original data $\X$ with probability greater than $1/2$.
We can describe the data process with the Markov chain $\X \rightarrow \X_\eta \rightarrow \Xh$, where $\Xh = \A(\X_\eta)$.
The mutual information of $\X, \X_\eta$ can be bounded by using the pairwise KL divergence bound \cite{yu1997assouad}.
\begin{align} \label{eq:PCA_pri_MI}
\MI[\X;\Xh] &\leq  \MI[\X;\X_\eta]\nonumber\\
& = n \MI[x_{ij};x_{\eta ij}]\nonumber\\
& \leq \frac{n}{|\{-1,+1\}|^2} \sum_{x_{ij}\in \{-1,+1\}} \sum_{x'_{ij}\in \{-1,+1\}}\KL(P_{x_{\eta ij}|x_{ij}}|P_{x_{\eta ij}|x'_{ij}})\nonumber\\
& = \frac{n}{|\{-1,+1\}|^2} \sum_{x_{ij}\in \{-1,+1\}} \sum_{x'_{ij}\in \{-1,+1\}}\KL(\N(x_{ij},\sigma_\eta^2)|\N(x'_{ij},\sigma_\eta^2))\nonumber\\
& = \frac{n}{|\{-1,+1\}|^2} \sum_{x_{ij}\in \{-1,+1\}} \sum_{x'_{ij}\in \{-1,+1\}} \frac{(x_{ij}-x'_{ij})^2}{2\sigma_\eta^2}\nonumber\\
& \leq \frac{n}{|\{-1,+1\}|^2} (|\{-1,+1\}|^2 - |\{-1,+1\}|) \frac{2}{\sigma_\eta^2}\nonumber\\
& \leq \frac{n}{\sigma_\eta^2}
\end{align}
Since $\X \in \XS = \{-1,+1\}^{n_1 \times n_2}$ where $n = n_1 n_2$ we have $|\XS| = 2^n$.
By Fano's inequality\cite{cover2012elements} and since $H(\X) \leq \log |\XS|$,
\begin{align*}
\P[\Xh \neq \X] &\geq 1- \frac{\MI(\X; \X_\eta) + \log 2}{\log |\XS|}\\
&\geq 1-\frac{\frac{n}{\sigma_\eta^2} + \log 2}{n\log 2}\\
\end{align*}
In order to have $P[\Xh \neq \X] \geq \gamma$, we require
    \begin{align*}
	    \frac{\frac{n}{\sigma_\eta^2}+\log 2}{n\log 2} &\leq 1-\gamma \\
    \end{align*}
    Thus, if $n > \frac{2}{1-\gamma}$, we have
	\begin{align*}
	     \sigma_\eta^2 &\geq \frac{1}{(1-\gamma-\frac{1}{n})\log 2} \\
	     \sigma_\eta^2 &\geq \frac{2}{(1-\gamma)\log 2}
	\end{align*}
\end{proof}

\subsection{Proof of Theorem~\ref{thm:PCA_priext}}
\begin{proof}
First, recall that $\Xh = \A(\X_\eta)$.
We proceed as in Theorem~\ref{thm:PCA_pri}, except for the Fano's inequality step.
We now use the Fano's inequality from \cite{duchi2013distance} together with eq.\eqref{eq:PCA_pri_MI} and the fact that $N_{\max}(t) = \binom{n}{t} \leq \left(\frac{ne}{t}\right)^t$.
Thus, 
\begin{align*}
	\P[d(\Xh, \X) > t] &\geq 1- \frac{\MI(\X;\Xh)+\log 2}{\log \left(\frac{|\XS|}{N_{\max}(t)}\right)}\\
	&\geq 1-\frac{\frac{n}{\sigma_\eta^2} + \log 2}{\log \left(\frac{2^n t^t }{(ne)^t}\right)}\\
	&= 1-\frac{\frac{n}{\sigma_\eta^2} + \log 2}{n\log 2+t(\log \frac{t}{n} - 1 )}\\
	\end{align*}
Note that $t\leq n$ in our analysis. 
In order to have $\P[d(\Xh, \X) > t] \geq \gamma$, we require 
\begin{align*}
    \frac{\frac{n}{\sigma_\eta^2} + \log 2}{n\log 2+t(\log \frac{t}{n} - 1 )} &\leq 1 - \gamma \\
    \sigma_\eta^2 &\geq \frac{n}{(1-\gamma) \left(n\log 2+t(\log \frac{t}{n} - 1 ) \right) -\log 2} \\
    \sigma_\eta^2 &\geq \frac{1}{(1-\gamma) \left(\log 2+\frac{t}{n}(\log \frac{t}{n} - 1 ) \right) -\frac{\log 2}{n}}
\end{align*}
Thus, if $n > \frac{2}{1-\gamma}$, we have
\begin{align*}
    \sigma_\eta^2 &\geq \frac{2}{(1-\gamma) \log 2+(1-\gamma)^2 t (\log \frac{t(1-\gamma)}{2} - 1 )}
\end{align*}
\end{proof}

\subsection{Proof of Theorem~\ref{rate:np}}
\begin{claim}[]\label{clm:nonparametric0}
    Let $\phi_1, \ldots, \phi_\infty$ be an infinitely dimensional orthonormal basis, and let $\vphi(\x) = ( \phi_1(\x),\ldots, \phi_\infty(\x))$.
    we represent the function $\ftheta : \XS \to \R$ by using the infinitely dimensional orthonormal basis.
    That is, $\ftheta(\x) = \sum_{j=1}^\infty{\nu\si{\ftheta}_j\phi_j(\x)} = \dotprod{\vnu\si{\ftheta}}{\vphi(\x)}$, where $\vnu\si{\ftheta} = (\nu\si{\ftheta}_1,\dots,\nu\si{\ftheta}_\infty)$.
    In the latter, the superindex $(\ftheta)$ allows for associating the infinitely dimensional coefficient vector $\vnu$ with the original function $\ftheta$.
    Then, we define the norm of the function $\ftheta$ with respect to the infinitely dimensional orthonormal basis.
    That is, $\|\ftheta\| = \|\vnu\si{\ftheta}\|$.
    
    Non-parametric generalized regression with fixed design fulfills Assumption~\ref{asm:closeloss} with probability at least $1-\delta$, scale function $c(\ftheta) = \|\ftheta\|$ and rate $\eps_{n,\delta}$, provided that the dual norm fulfills $\|\frac{1}{n}\sum_i{(\psi(y\si{i}, \eta\si{i})-\E_{y \sim \DD_i, \eta \sim \ND_i}[\psi(y\si{i}, \eta\si{i})])\vphi(\x\si{i}})\|_*  \\ \leq \eps_{n,\delta}$
	
	This problem also fulfills Assumption~\ref{asm:difference} with $\eps'_n =0$.
\end{claim}
\begin{proof}
    We first show that $\Ln(\vtheta) = \L(\vtheta)$.
	We have
	\begin{align*}
	\Ln (\vtheta)
	&= \E_{(\forall i)y\si{i}\sim \DD\si{i}, \eta\si{i} \sim \ND}[\frac{1}{n}\sum_{i}-\psi(y\si{i},\eta\si{i})\theta(\x\si{i}) 
	+ \log\mathcal{Z}(\theta(\x\si{i}))] \\
	& = \E_{(\forall i)y\si{i}\sim \DD\si{i}}[\frac{1}{n}\sum_{i}-\E_{\eta\si{i} \sim \ND}[\psi(y\si{i},\eta\si{i})]\theta(\x\si{i}) 
	+ \log\mathcal{Z}(\theta(\x\si{i}))] \\
	& = \E_{(\forall i)y\si{i}\sim \DD\si{i}}[\frac{1}{n}\sum_{i}-t(y\si{i})\theta(\x\si{i}) 
	+ \log\mathcal{Z}(\theta(\x\si{i}))]\\
	& = \L(\vtheta)
	\end{align*}
	For proving that Assumption~\ref{asm:difference} holds, note that $\Ln(\vtheta) = \L(\vtheta)$ for any $\vtheta$, and thus $\eps'_n = 0$.
	
	For proving that Assumption~\ref{asm:closeloss} holds, we have for all $\vtheta$
    \begin{align*}
	    |\Lhn(\vtheta) - \Ln(\vtheta)| 
	    &= |\frac{1}{n}\sum_{i} \psi(y\si{i}, \eta\si{i})\theta(\x\si{i}) - 
	    \frac{1}{n}\sum_{i} \E_{\DD, \ND}[\psi(y\si{i}, \eta\si{i})]\theta(\x\si{i})| \\
	    &= |\dotprod{\frac{1}{n}\sum_{i}(\psi(y\si{i}, \eta\si{i})-\E_{\DD, \ND}[\psi(y\si{i}, \eta\si{i})])\vphi(\x\si{i})}{\vnu\si{\theta}} | \\ 
	    &\leq \|\frac{1}{n}\sum_{i}(\psi(y\si{i}, \eta\si{i})-\E_{\DD, \ND}[\psi(y\si{i}, \eta\si{i})])\vphi(\x\si{i})\|_*\|\vnu\si{\theta}\| \\
	    &\leq \eps_{n,\delta}\|\vtheta\|
	\end{align*}
\end{proof}
Let $\x \in \XS = \R^p$.
Let $\|\cdot\|_* = \|\cdot\|_\infty$ and $\|\cdot\| = \|\cdot\|_1$.
Let $(\forall \x){\rm\ }\|\vphi(\x)\|_* \leq B$ and thus $(\forall ij){\rm\ }|\phi_j(\x\si{i})| \leq B$.
The complexity of our nonparametric model grows with more samples.
Assume that we have $q_n$ orthonormal basis functions $\varphi_1,\dots,\varphi_{q_n} : \R \to \R$.
Let $q_n$ be increasing with respect to the number of samples $n$.
With these bases, we define $q_n p$ orthonormal basis functions of the form $\phi_j(\x) = \varphi_k(x_l)$ for $j=1,\dots,q_n p$, $k=1,\dots,q_n$, $l=1,\dots,p$.
According to Lemma~\ref{lemma:subgaussian} and Lemma~\ref{lemma:finitevariance}, the variance of  $\psi(y, \eta)$ is $\sigma^2 = \sigma^2_y + \sigma^2_\eta$.
We now focus on proving that $\|\frac{1}{n}\sum_{i}(\psi(y\si{i}, \eta\si{i})-\E_{\DD, \ND}[\psi(y\si{i}, \eta\si{i})])\vphi(\x\si{i})\|_* \leq \varepsilon_{n,\delta}$ which is the precondition of Claim iv.

\paragraph{Claim~\ref{clm:nonparametric0}  Sub-Gaussian case with $\ell_1$-norm.}
Let $\forall i, \psi(y\si{i},\eta\si{i})$ be sub-Gaussian with parameter $\sigma$.
Therefore $\forall i, \psi(y\si{i},\eta\si{i})\phi_j(\x\si{i})$ is sub-Gaussian with parameter $\sigma B$. By Claim~\ref{clm:nonparametric0} , and by the union bound, sub-Gaussianity and independence, 
\begin{align*}
    & \P[\|\frac{1}{n}\sum_{i}(\psi(y\si{i}, \eta\si{i})-\E_{\DD, \ND}[\psi(y\si{i}, \eta\si{i})])\vphi(\x\si{i})\|_*>\eps]\\
	&= \P[(\exists j){\rm\ }|\frac{1}{n}\sum_i(\psi(y\si{i},\eta\si{i})-
	\E_{\DD, \ND}[\psi(y\si{i}, \eta\si{i})])\phi_j(\x\si{i})| > \eps] \\
	 &\leq 2q_n p{\rm\ }\lexp{-\frac{n\eps^2}{2(\sigma B)^2}} = \delta
\end{align*}
 By solving for $\eps$, we have $\eps_{n,\delta} = \sigma B \sqrt{\sfrac{2}{n}(\log{p} + \log{q_n} + \log\sfrac{2}{\delta})}$.
 
 \paragraph{Claim~\ref{clm:nonparametric0}  Finite variance case with $\ell_1$-norm.}
 Let $\forall i, \psi(y\si{i},\eta\si{i})$ have variance at most $\sigma^2$.
 Therefore $\forall i, \psi(y\si{i},\eta\si{i})\phi_j(\x\si{i})$ has variance at most $(\sigma B)^2$. By Claim~\ref{clm:nonparametric0}, and by the union bound and Chebyshev's inequality,
\begin{align*}
	& \P[\|\frac{1}{n}\sum_{i}(\psi(y\si{i}, \eta\si{i})-\E_{\DD, \ND}[\psi(y\si{i}, \eta\si{i})])\vphi(\x\si{i})\|_*>\eps]\\
	&= \P[(\exists j){\rm\ }|\frac{1}{n}\sum_i(\psi(y\si{i},\eta\si{i})-
	\E_{\DD, \ND}[\psi(y\si{i}, \eta\si{i})])\phi_j(\x\si{i})| > \eps]\\ &\leq q_n p{\rm\ }\frac{(\sigma B)^2}{n\eps^2} = \delta
\end{align*}
By solving for $\eps$, we have $\eps_{n,\delta} = \sigma B \sqrt{\frac{q_n p}{n\delta}}$.

\subsection{Proof of Theorem~\ref{rate:max-margin}}
\begin{claim}[]\label{clm:max-margin}
	Max-margin matrix factorization fulfills Assumption~\ref{asm:closeloss} with probability 1, scale function $c(\vtheta) = \|\vtheta\|_1$ and rate $\eps_{n,\delta} = \O(\frac{1}{n})$.
	Furthermore, max-margin matrix factorization fulfills Assumption~\ref{asm:difference}  with $\eps'_n = \frac{2K(1-q)}{n}$ and $c(\vtheta) = \|\vtheta\|_1$.
\end{claim}
\begin{proof}
    To prove this problem fulfills Assumption~\ref{asm:closeloss}, we have:
    \begin{align*}
       & |\Lhn(\vtheta) - \Ln(\vtheta)| \\
       & = |\frac{1}{n}\sum_{ij}(f(x_{ij}\eta_{ij}\theta_{ij}) - \E_{\DD, \ND}[f(\x_{ij}\eta_{ij}\theta_{ij})])|\\
       & = |\frac{1}{n}\sum_{ij}(1[x_{ij}\eta_{ij}=+1]f(\theta_{ij}) + 1[x_{ij}\eta_{ij}=-1]f(-\theta_{ij}) - 
       \P[x_{ij}\eta_{ij}=+1]f(\theta_{ij}) - \P[x_{ij}\eta_{ij}=-1]f(-\theta_{ij}))| \\
       &= |\frac{1}{n}\sum_{ij}((1[x_{ij}\eta_{ij}=+1]-\P[x_{ij}\eta_{ij}=+1]) f(\theta_{ij}) +
       (1[x_{ij}\eta_{ij}=-1]-\P[x_{ij}\eta_{ij}=-1])f(-\theta_{ij}))| \\
       &\leq \frac{1}{n}\sum_{ij}(|1[x_{ij}\eta_{ij}=+1]-\P[x_{ij}\eta_{ij}=+1]||f(\theta_{ij})| + 
       |1[x_{ij}\eta_{ij}=-1]-\P[x_{ij}\eta_{ij}=-1]||f(-\theta_{ij})|\\
       & \leq \frac{1}{n}\sum_{ij}2K|\theta_{ij}| \\
       & = \frac{2K}{n}\|\vtheta\|_1
    \end{align*}
    To prove this problem fulfills Assumption~\ref{asm:difference}.
    Let $K$ be the Lipschitz constant of $f$.
    Note that
    \begin{align*}
        \E_{\veta}[\Lhn(\vtheta)] = \frac{1}{n}\sum_{ij} q f(x_{ij}\theta_{ij}) + (1-q) f(-x_{ij}\theta_{ij})
    \end{align*}
    Thus, we have
    \begin{align*}
        |\Lh(\vtheta) - \E_{\ND}[\Lhn(\vtheta)]| 
        &=|\frac{1}{n}\sum_{ij}f(x_{ij}\theta_{ij}) - 
        (\frac{1}{n}\sum_{ij} q f(x_{ij}\theta_{ij}) + (1-q) f(-x_{ij}\theta_{ij})) |\\
        &= \frac{1}{n}|\sum_{ij} (1-q) f(x_{ij}\theta_{ij}) - (1-q) f(-x_{ij}\theta_{ij})| \\
        &= \frac{(1-q)}{n}|\sum_{ij}f(x_{ij}\theta_{ij}) - f(-x_{ij}\theta_{ij})|\\
        & \leq \frac{(1-q)}{n}|\sum_{ij} 2K(x_{ij}\theta_{ij}) |\\
        & \leq \frac{2K(1-q)}{n}\|\vtheta\|_1
    \end{align*}
    By Jensen's inequality:
    \begin{align*}
         |\L(\theta) - \Ln(\theta)|
        & = |\E_{\DD} [ \Lh(\theta) - \Lhn(\theta) ] |\\
        &\leq \E_{\DD} | \Lh(\theta) - \Lhn(\theta) | \\
        & \leq \frac{2K(1-q)}{n}\|\vtheta\|_1
    \end{align*}
\end{proof}

\subsection{Proof of Theorem~\ref{thm:max-margin_pri}}
\begin{proof}
    Using Fano's inequality, we show that it will be impossible to recover the original data $\X$ with probability greater than $1/2$.
    We can describe the data process with the Markov chain $\X \rightarrow \X_\eta \rightarrow \Xh$, where $\Xh = \A(\X_\eta)$.
	Let $\B(q)$ denote the probability distribution that returns $+1$ with probability $q$ and $-1$ with probability $1-q$.
    Note that since $x_{ij} \in \{-1,+1\}$ and $\eta_{ij} \sim \B(q)$, then $\eta_{ij}x_{ij} \sim \B(1/2+(q-1/2)x_{ij})$.
    The mutual information of $\X, \X_\eta$ can be bounded by using the pairwise KL divergence bound \cite{yu1997assouad}.
    \begin{align} \label{eq:max-margin_pri_MI}
        \MI[\X; \Xh] &\leq \MI[\X; \X_\eta] \nonumber\\
                &= n\MI[x_{ij}; \eta_{ij}x_{ij}] \nonumber\\
                & \leq \frac{n}{|\{-1,+1\}|^2} \sum_{x_{ij}\in \{-1,+1\}} \sum_{x'_{ij}\in \{-1,+1\}}\KL(P_{\eta_{ij}x_{ij}|x_{ij}}|P_{\eta_{ij}x'_{ij}|x'_{ij}})\nonumber\\
                & = \frac{n}{|\{-1,+1\}|^2} \sum_{x_{ij}\in \{-1,+1\}} \sum_{x'_{ij}\in \{-1,+1\}}\KL(\B(1/2+(q-1/2)x_{ij})|\B(1/2+(q-1/2)x'_{ij})\nonumber\\
				& = \frac{n}{|\{-1,+1\}|^2} (|\{-1,+1\}|^2 - |\{-1,+1\}|) (q\log \frac{q}{1-q} + (1-q)\log \frac{1-q}{q})\nonumber\\
                & \leq \frac{n}{2}(q\log \frac{q}{1-q} + (1-q)\log \frac{1-q}{q})\nonumber\\
                &= \frac{n}{2}(2q-1)\log \frac{q}{1-q}
    \end{align}
    Since $\X \in \XS = \{-1,+1\}^{n_1 \times n_2}$ where $n = n_1 n_2$ we have $|\XS| = 2^n$.
    By Fano's inequality\cite{cover2012elements} and since $H(\X) \leq \log |\XS|$,
    \begin{align*}
        P[\Xh \neq \X] &\geq 1- \frac{\MI[\X; \Xh] + \log 2}{\log |\XS|}\\
        &\geq 1 - \frac{\frac{n}{2}(2q-1)\log \frac{q}{1-q} + \log 2}{n \log 2}
    \end{align*}
    In order to have $P[\Xh \neq \X] \geq \gamma$, we require
    \begin{align*}
        \frac{\frac{n}{2}(2q-1)\log \frac{q}{1-q} + \log 2}{n \log 2} &< 1- \gamma\\
        (2q-1)\log \frac{q}{1-q} &< 2(1- \gamma -\frac{1}{n})\log 2\\
    \end{align*}
    Note that 
    \begin{align*}
        (2q-1)\log \frac{q}{1-q} &< (2q-1)(\frac{q}{1-q} - 1)\\
        &= \frac{4q^2-4q+1}{1-q}
    \end{align*}
    Let $g=2(1- \gamma -\frac{1}{n})\log 2$ and $g\in(0,2\log 2)$, we can solve
    \begin{align*}
        \frac{4q^2-4q+1}{1-q} &< g\\
    \end{align*}
    Solving the inequality above, we get, $q\in(\frac{1}{2}, \frac{1}{2}+\frac{-g+\sqrt{g(g+8)}}{8})$.
    A sufficient condition for the latter is $q\in(\frac{1}{2}, \frac{1}{2}+\frac{1-\gamma-\frac{1}{n}}{4})$, as $\frac{g}{\log 2} < -g + \sqrt{g(g+8)}$ for $g\in(0,2\log 2)$.
    If we further assume that $n > \frac{2}{1-\gamma}$, we can have $q \in (\frac{1}{2}, \frac{1}{2} + \frac{(1-\gamma)}{8})$.
\end{proof}

\subsection{Proof of Theorem~\ref{thm:max-margin_priext}}
\begin{proof}
First, recall that $\Xh = \A(\X_\eta)$.
We proceed as in Theorem~\ref{thm:max-margin_pri}, except for the Fano's inequality step.
We now use the Fano's inequality from \cite{duchi2013distance} together with eq.\eqref{eq:max-margin_pri_MI} and the fact that $N_{\max}(t) = \binom{n}{t} \leq \left(\frac{ne}{t}\right)^t$.
Thus, 
\begin{align*}
	\P[d(\Xh, \X) > t] &\geq 1- \frac{\MI(\X;\Xh)+\log 2}{\log \left(\frac{|\XS|}{N_{\max}(t)}\right)}\\
	&\geq 1-\frac{\frac{n}{2}(2q-1)\log \frac{q}{1-q} + \log 2}{\log \left(\frac{2^n t^t }{(ne)^t}\right)}\\
	&= 1-\frac{\frac{n}{2}(2q-1)\log \frac{q}{1-q} + \log 2}{n\log 2+t(\log \frac{t}{n} - 1 )}\\
	\end{align*}
Not that $t\leq n$ in our analysis. 
In order to have $\P[d(\Xh, \X) > t] \geq \gamma$, we require 
\begin{align*}
    \frac{\frac{n}{2}(2q-1)\log \frac{q}{1-q} + \log 2}{n\log 2+t(\log \frac{t}{n} - 1 )} &\leq 1 - \gamma \\
    (2q-1)\log \frac{q}{1-q} &\leq \frac{(1-\gamma) \left(n\log 2+t(\log \frac{t}{n} - 1 ) \right) -\log 2}{n/2}\\
    (2q-1)\log \frac{q}{1-q} &\leq 2(1-\gamma) \left(\log 2+\frac{t}{n}(\log \frac{t}{n} - 1 ) \right) -\frac{2\log 2}{n}
\end{align*}
Let $G_{\gamma, n, t}=(1-\gamma) \left(\log 2+\frac{t}{n}(\log \frac{t}{n} - 1 ) \right) -\frac{\log 2}{n}$.
A reasoning similar to the proof of Theorem~\ref{thm:max-margin_pri} leads to $q\in (\frac{1}{2}, \frac{1}{2}+\frac{-G_{\gamma, n, t}+\sqrt{G_{\gamma, n, t}(G_{\gamma, n, t} + 8)}}{8})$.
\end{proof}

\section{Irrecoverability Versus Privacy in Our Examples}
\label{app:relation_dp}

\subsection{A General Privacy Example}
\label{app:privacy_example}
We invoke Definition~\ref{def:dp} for sets $S$ of size 1. 
In this case we have $S = \{z\}$  for $z \in \mathcal{Z}$ , and  therefore $\M(x) \in S$ is equivalent to $\M(x) = z$.
Furthermore, we use datasets $x$ and $x'$ that differ in $\alpha n$ samples, where $\alpha \in (0,1]$ is constant with respect to $n$.
We believe this regime is fair for comparison, since Theorem~\ref{thm:irrecov} uses privacy for arbitrary datasets $x$ and $x'$ (i.e., $\alpha = 1$). 
Furthermore, regimes such as differential privacy (where datasets $x$ and $x'$ differ in a single data point) assume that the attacker knows all samples except one (which is an irrelevant regime for recoverability where the attacker does not know any of the samples).

Let $\DD_x$ be the domain of samples, denote a dataset with $n$ samples as $x \in \XS \equiv \DD_x^n$.
Let $z \in \mathcal{Z}\equiv\DD_z^n$ be the perturbed version of $x$ (i.e., $z_i$ is the noisy observation of $x_i$ for $i \in [n]$).
Let $\mathcal{M}$ be the perturbation algorithm that takes $x$ as input and returns $z$ as output. 
We have:
\begin{align*}
    \P_{\M}[\M(x)=z] &\leq e^\epsilon \P[\M(x') = z] + \delta, 
\end{align*}
which is equivalent to
\begin{align*}
     p(z|x) &\leq e^\epsilon p(z|x') + \delta.
\end{align*}
By independence, we have
\begin{align*}
    \prod_{i=1}^{n}p(z_i|x_i) &\leq e^\epsilon\prod_{i=1}^{n}  p(z_i|x'_i)  + \delta
\end{align*}
Now $\forall i \in [n], p(z_i|x_i) \leq e^{\frac{\epsilon}{\alpha n}}  p(z_i|x'_i) + \frac{\delta}{n}$ is a sufficient condition to satisfy privacy.

\subsection{Irrecoverable Example in Section~\ref{subsec:MLE}}
Following Section~\ref{app:privacy_example},
this example satisfies Definition~\ref{def:dp} if $\tilde{\sigma}_\eta \geq \frac{\alpha n\sqrt{8 \sqrt{p}\log \frac{1.25}{\delta}}}{\epsilon} $ for two datasets with $\alpha n$ different samples,   because of the additive Gaussian noise  and robustness of post-processing ~\cite{dwork2014algorithmic}.
If the variance $\sigma^2_\eta \in \left[\frac{4}{(1-\gamma)\log 2}, \tilde{\sigma}^2_\eta \right)$, the example is not private but irrecoverable.

\subsection{Irrecoverable Example in Section~\ref{subsec:glm}}
Following Section~\ref{app:privacy_example} and similar to the previous example, this example satisfies Definition~\ref{def:dp} if $\tilde{\sigma}_\eta  \geq \frac{\alpha n\sqrt{8 \log \frac{1.25}{\delta}}}{\epsilon} $ for two datasets with $\alpha n$ different samples, because of the additive Gaussian noise ~\cite{dwork2014algorithmic}.
If the variance $\sigma^2_\eta \in \left[\frac{8}{(1-\gamma)\log 2}, \tilde{\sigma}^2_\eta \right)
$, the example is not private but irrecoverable.

\subsection{Irrecoverable Example in Section~\ref{subsec:pca}}
Following Section~\ref{app:privacy_example},
similar to the previous examples, this example satisfies Definition~\ref{def:dp} if $\tilde{\sigma}_\eta  \geq \frac{\alpha n\sqrt{8 \log \frac{1.25}{\delta}}}{\epsilon} $ for two datasets with $\alpha n$ different samples, because of the additive Gaussian noise~\cite{dwork2014algorithmic}.
If the variance $\sigma^2_\eta \in \left[\frac{8}{(1-\gamma)\log 2}, \tilde{\sigma}^2_\eta \right)$, the example is not private but irrecoverable.

\subsection{Irrecoverable Example in Section~\ref{subsec:max-margin}}
Following Section~\ref{app:privacy_example},
this example satisfies $(\epsilon,0)$-privacy from Definition~\ref{def:dp} for the matrix $\X \in \{-1, 1\}^{n_1 \times n_2}$, $n=n_1 n_2$ if we have $\tilde{q} = \frac{e^{\frac{\epsilon}{\alpha n}}}{e^{\frac{\epsilon}{\alpha n}}+1}$ for two datasets with $\alpha n$ different samples.
If $q \in \left[\frac{1}{2}, \frac{1}{2} + \frac{(1-\gamma)\log 2}{8} \right]$ and $q > \tilde{q}$, the example is not private but irrecoverable.

\end{document}